\documentclass[final,12pt]{alt2024} 

\title{The Dimension of Self-Directed Learning}

\usepackage{times}
\usepackage[utf8]{inputenc} 
\usepackage[T1]{fontenc}    
\usepackage{hyperref}       
\usepackage{url}            
\usepackage{booktabs}       
\usepackage{amsfonts}       
\usepackage{nicefrac}       
\usepackage{microtype}      
\usepackage{xcolor}         
\usepackage{float}
\usepackage{amsmath}
\usepackage{amssymb}
\usepackage{bbm}
\usepackage{algorithm}
\usepackage{algorithmic}

\DeclareMathOperator*{\argmin}{arg\,min}

\usepackage{comment}
\usepackage{tikz}

\usepackage{enumitem}

\altauthor{%
 \Name{Pramith Devulapalli} \Email{pdevulap@purdue.edu}\\
 \addr Department of Computer Science, Purdue University
 \AND
 \Name{Steve Hanneke} \Email{steve.hanneke@gmail.com}\\
 \addr Department of Computer Science, Purdue University%
}

\begin{document}

\maketitle

\begin{abstract}
Understanding the self-directed learning complexity has been an important problem that has captured the attention of the online learning theory community since the early 1990s.
Within this framework, the learner is allowed to adaptively choose its next data point in making predictions unlike the setting in adversarial online learning.

In this paper, we study the self-directed learning complexity in both the binary and multi-class settings, and we develop a dimension, namely $SDdim$, that exactly characterizes the self-directed learning mistake-bound for any concept class.
The intuition behind $SDdim$ can be understood as a two-player game called the ``labelling game''.
Armed with this two-player game, we calculate $SDdim$ on a whole host of examples with notable results on axis-aligned rectangles, VC dimension $1$ classes, and linear separators.
We demonstrate several learnability gaps with a central focus on self-directed learning and offline sequence learning models that include either the best or worst ordering.
Finally, we extend our analysis to the self-directed binary agnostic setting where we derive upper and lower bounds.
\end{abstract}


\section{Introduction}

Imagine a setting where you are asked to solve a 1000-piece jigsaw puzzle designed by your friend. 
One-by-one, you select an unused piece and place it where you think it appropriately belongs.
If the piece is misplaced, your friend will tell you its correct position.
After knowing the right answer, you select another unused piece and repeat this cycle until you solve the jigsaw puzzle.
Your goal is to make the least number of mistakes while solving the puzzle.
Which parts of the puzzle do you prioritize? 
What order of pieces do you select and predict positions for?
How do you gauge the current condition of the jigsaw puzzle to determine the next piece chosen?

These nuanced questions highlight the necessity to study a different variation of online learning.
If one were to view yourself as a learning algorithm, then choosing a different piece at each stage of the puzzle represents adaptively selecting a sequence of instances, the predicted and correct positions of the pieces are the predicted and true labels respectively, and the mistakes made while solving the puzzle resemble a mistake-bound. 
Such an online learning model, described as \textit{self-directed learning},
allows the learner, from a specified subset of the instance space, to choose its next example to predict on.
After each prediction, the learner observes the true label and the cycle repeats.
The goal of the self-directed learner boils down to making predictions on an adaptively chosen sequence of instances to minimize mistakes.
The study of self-directed learning can be traced back to \citet{goldman1994power} who framed it as a natural variation of the standard membership query model proposed by \citet{angluin:87} where queries are only counted if the prediction is incorrect.
\citet{goldman1994power} launched a new line of research into the self-directed complexity by analyzing it on a wide variety of concept classes.
Then, \citet{ben1995self} extended this study by formally deriving self-directed mistake-bounds and resolving open questions between the self-directed complexity and the VC-dimension. 

How does self-directed learning fit within the larger scope of online learning theory?
Observing the past couple decades of online learning literature, it's clear that its highly popular counterpart, adversarial online learning, has dominated attention within this field.
The seminal work by Littlestone \citep{littlestone:88} introduced a combinatorial characterization that captures the mistake-bound of adversarial online learning under the binary realizable setting.
Inspired by the Littlestone dimension, characterizing the mistake-bounds of different online learning frameworks garnered much attention in the mid-1990s.
\citet{ben1995self,ben1997online} introduced further variations of the online learning model which are termed as offline learning models. 
In these models, the entire sequence of instances is either chosen by the learner or adversary before the prediction process starts.
To understand their relationship to the self-directed model, \citet{ben1995self,ben1997online} formulated their mistake-bounds to understand the power of adaptiveness in choosing the sequence of instances "on-the-fly".

In this paper, we study the self-directed learning complexity in both the binary and multi-class settings, and we develop a dimension, namely $SDdim$, that exactly characterizes the self-directed learning mistake-bound for any concept class. 
Our key objective in studying the self-directed learning setting lies in understanding the behavior within the multi-class setting to create a complete theory of self-directed learning under the realizable setting. 
Another primary interest is to find a neat combinatorial characterization of the optimal mistake-bound for self-directed learning which is defined as $M_{sd}(H)$.
The quantity $M_{sd}(H)$ can be understood as the best (over all algorithms) worst-case mistake-bound for the concept class $H$ (over all subsets of the domain $\mathcal{X}$ and any choice of the target concept).

We have organized the paper into the following sections.
In Section $2$, we formally setup the self-directed scenario in the multi-class setting.
Section $3$ describes a two-player game called the \textit{labelling game} we designed to easily compute $SDdim$.
The main result of our paper is stated in the following Theorem $1$ with Section $4$ devoted to its proof:
\begin{theorem}
    $SDdim(H) = M_{sd}(H)$.
    \label{thm:SDdim=m_sd}
\end{theorem}
In Section $5$, we showcase our dimension $SDdim$ by computing it on a wide variety of learning theoretic examples and exploring learnability gaps.
Finally in Section $6$, we extend our analysis of self-directed learning to the binary agnostic setting.
To tie in the central themes presented, we summarize the following contributions of our paper:
\begin{itemize}
    \item We provide the first formal study of self-directed learning under a multi-class classification framework.
    \item We develop a two player game termed as the \textit{labelling game} that accurately captures the learning complexity in the self-directed case.
    \item We construct a game-tree that represents an efficient characterization of the exact mistake-bound model for self-directed learning that holds for all subsets $S \subseteq \mathcal{X}$.
    \item We propose an algorithm called the SD-SOA that makes a number of mistakes bounded by the quantity $SDdim(H)$.
    \item We compare the learning complexity of self-directed learning with other forms of online learning, offline learning, and the PAC-learnable case by providing both multi-class and binary examples.
    \item We showcase a straightforward computation of $M_{sd}(H_d) = 2d$ where $H_d$ corresponds to the concept class of axis-aligned rectangles in $\mathbb{R}^d$. 
    \item On $VC(H)=1$ classes, we present a direct approach in proving $M_{sd}(H)=1$.  
    \item We construct a concept class $H$ exhibiting a learnability gap between $M_{sd}(H)$ and $M_{best}(H)$ in the multi-class setting where $M_{best}$ is the offline learning model allowing the learner to choose the entire sequence of instances before the prediction process starts.
    \item We extend our results to the agnostic setting and analyze the expected regret with an upper bound of $O\left(\sqrt{VC(H)T\ln(T)}\right)$ and a lower bound of $\Omega\left(\sqrt{VC(H)T}\right)$ respectively.
\end{itemize}

\section{Formal Setup}

\subsection{Basic Definitions}

In this section, we formally present the notion of the self-directed learning mistake-bound. This framework is quite similar to the online learning framework used by Littlestone \citep{littlestone:88} and Ben-David et al. \citep{ben1995self}. 
We define $\mathcal{X}$ and $\mathcal{Y}$ to be arbitrary, non-empty sets where $\mathcal{X}$ is referred to as the instance space and $\mathcal{Y}$ is the label space.
Throughout the paper, we assume that $|\mathcal{Y}| \geq 2$.
We define $H$ to be an arbitrary, non-empty set of functions from $\mathcal{X} \rightarrow \mathcal{Y}$ called the concept class. 
In the special case of $\mathcal{Y} = \{0, 1\}$, we refer to this as binary classification.

\subsubsection{Self-Directed Setting}

In this section, we describe in detail the learning environment. 
Assuming a non-empty $H$ and a finite subset $S \subseteq \mathcal{X}$, one full round in the self-directed learning follows the sequence described below:
\begin{enumerate}
    \item The learner chooses some $x \in S$. 
    \item The learner makes a prediction on $x$ by selecting a label $y \in \mathcal{Y}$.
    \item The learner receives the true label from the environment.
    \item Update $S = S \setminus \{ x \}$.
\end{enumerate}

The next round then proceeds in the same exact fashion. 
Self-directed learning follows this iterative cycle until the learner has labelled all the points in $S$.
In this paper (up to Section $\ref{section:agnostic}$), we focus on realizable self-directed learning which assumes that there exists a function $f^* \in H$ that is consistent with the sequence of instances and true labels.

\subsubsection{Self-Directed mistake-bound}

A self-directed learning algorithm $\mathcal{A}$ operating on some finite input space $S \subseteq \mathcal{X}$ can adaptively choose a sequence of points to query.
Unlike the adversarial online learning model where the learner is provided instances by the environment to predict upon, the self-directed algorithm chooses its instances ``on-the-fly'' so the query sequence can depend on the labels of previous instances.
Since the self-directed learning model is studied under the realizable setting, there exists a target concept consistent with the entire sequence of points and labels. 
The number of mistakes that the self-directed learning algorithm $\mathcal{A}$ makes given a finite input space $S$ and a target concept $f^*$ is denoted as $M_{sd}(\mathcal{A}, S, f^*)$.
Then, the mistake-bound of algorithm $\mathcal{A}$ on a fixed $S$ for the concept class $H$ is denoted as 
\begin{align*}
    M_{sd}(\mathcal{A}, H, S) = \max_{f^* \in H} M_{sd}(\mathcal{A}, S, f^*).
\end{align*}

\noindent For completeness, we define $M_{sd}(\mathcal{A}, S, H) = -1$ when $H = \emptyset$.
Then, we consider the self-directed mistake-bound on $H$ and some finite $S \subseteq \mathcal{X}$ by taking the minimum over all learning algorithms $\mathcal{A}$.
\begin{align*}
    M_{sd}(H, S) = \min_{\mathcal{A}} M_{sd}(\mathcal{A}, H, S) = \min_{\mathcal{A}} \max_{f^* \in H} M_{sd}(\mathcal{A}, S, f^*)
\end{align*}
\noindent We define the optimal mistake-bound $M_{sd}(H)$ as the maximum over all finite subsets in $\mathcal{X}$.
\begin{align*}
    M_{sd}(H) = \max_{S \subseteq \mathcal{X}, |S| < \infty} M_{sd}(H, S)
\end{align*}
If $H = \emptyset$, then $M_{sd}(H) = -1$. 
Additionally, for any non-empty $H$ and $S = \emptyset$, $M_{sd}(H, \emptyset) = 0$ since any $h \in H$ is realizable on $S$.
Colloquially, $M_{sd}$ is simply referred to as the mistake-bound of self-directed learning.

\subsection{Relation to Previous Literature}

A central theme within our work is to understand the worst-case mistake-bound $M_{sd}(H)$ over all finite subsets $S \subseteq \mathcal{X}$.
Prior work within the self-directed literature has exclusively focused on characterizing the subset-dependent mistake-bound $M_{sd}(H, S)$. 
For example, the initial work of \citeauthor{goldman1994power} (\citeyear{goldman1994power}, Theorem $4$) proved that $M_{sd}(H^d_n, S^d_n) = 2$ where $S^d_n$ represents a $d$-dimensional grid containing $n^d$ equally spaced points and $H^d_n$ is the concept class of axis-aligned rectangles built on $S^d_n$.
In Section \ref{section:rectangles}, we show that $M_{sd}(H^d_n) = VC(H^d_n)$ when considering all finite subsets $S$.
Similarly, \citeauthor{ben1995self} (\citeyear{ben1995self}, Theorem $9$) found an example of an $(H, S)$ separating $M_{sd}(H, S)$ and $M_{best}(H, S)$, where $M_{best}(H, S)$ is the mistake-bound of the non-adaptive version of self-directed learning where the order of examples is finalized before the prediction process starts.
However, in the worst-case analysis over all finite $S$, the $H$ in that result does not yield a separation between $M_{sd}(H)$ and $M_{best}(H)$.
In contrast, our Theorem \ref{thm:m_best_learnability} constructs an $H$ that shows a separation between $M_{sd}(H)$ and $M_{best}(H)$ (specifically for the multi-class setting).

A central theme within our work is to understand the worst-case mistake-bound $M_{sd}(H)$ over all finite subsets $S \subseteq \mathcal{X}$.
\citet{ben1995self} characterized different modes of online learning through the rank of different families of binary trees which require as input a finite subset $S \subseteq \mathcal{X}$. 
In the case of adversarial online learning, one can compute the worst-case mistake-bound by finding the maximum rank of these binary trees over all possible, finite subsets $S \subseteq \mathcal{X}$.
However, the Littlestone dimension of a concept class $H$, or $LD(H)$, is a direct characterization of this exact worst-case mistake-bound that solely relies on $H$.
In a similar vein, our motivation is to discover a combinatorial characterization, $SDdim$, that directly computes $M_{sd}(H)$ and is easy-to-use in practice.

We also the extend the research on self-directed learning to the multi-class setting; previous literature on self-directed learning was only considered in the binary class setting.
Additionally, we provide an analysis of the self-directed setting under the binary agnostic setting where we prove upper and lower bounds on the regret.

\section{Labelling Game}

To tackle the puzzle of self-directed learning in the multi-class setting, we developed a two-player game such that the game dynamics mimic the self-directed mistake-bound $M_{sd}(H, S)$.
Our two-player game, the labelling game, is played between Player A, the adversary, and Player B, the learner, where the number of rounds is synonymous with the value of $M_{sd}(H, S)$.
Assume that the labelling game is executed on some non-empty concept class $H$ and finite subset $S \subseteq \mathcal{X}$ with the label set $\mathcal{Y}$.
In the binary setting, the labelling game works in the following fashion:
\begin{enumerate}
    \item Player A is allowed to label some set of points $C \subseteq S$. Update $S = S \setminus C$.
    \item Player B is then allowed to label any single point from $S$. Without loss of generality, let $x \in S$ be the point Player B has chosen with the label of $y$. Update $S = S \setminus \{ x \}$.
    \item Repeat steps 1-2 until the entire dataset is labelled. 
\end{enumerate}

\noindent In the multi-class setting, the labelling game works in the following fashion:
\begin{enumerate}
    \item Player A is allowed to label some set of points $C \subseteq S$. Update $S = S \setminus C$.
    \item For each $x \in S$, Player A designates two labels $y^x_1, y^x_2 \in \mathcal{Y}$ where $y^x_1 \neq y^x_2$.
    \item Player B is then allowed to label any single point from $S$ with one of the two designated labels. Without loss of generality, let $x \in S$ be the point Player B has chosen with the label of $y^x_1$. Update $S = S \setminus \{ x \}$.
    \item Repeat steps 1-3 until the entire dataset is labelled. 
\end{enumerate}

In the labelling game, the objective of Player B is to create some unrealizable sequence of points and labels as fast as possible whereas Player A wants to execute the game for as long as possible. 
Therefore, Player A is incentivized to select points and labels that are realizable with respect to $H$ and eliminate any favorable choices for Player B.
The game terminates when all of $S$ has been labelled.      
The crucial part of this game lies in the number of rounds Player A can extend the game against Player B's actions.
A round is defined as a complete cycle between Player A and Player B.
The payout of the game is defined as the number of rounds the labelling game takes on until the entire dataset has been labelled .
If the labelling of the dataset is unrealizable, then the payout of the game is equal to $-1$.
$SDdim(H, S)$ is then defined as the minimax value of the labelling game's payout on $H$ given the finite subset $S \subseteq \mathcal{X}$.
Naturally, we let $SDdim(H) = \max_{S \subseteq \mathcal{X}, |S| < \infty} SDdim(H, S)$.
If $S = \emptyset$, then there are no points for Player $A$ to select so a full round cannot even be completed and $SDdim(H, \emptyset) = 0$.
For the case that $H$ is an empty concept class, we define $SDdim(\emptyset, S) = -1$ for any $S \subseteq \mathcal{X}$ since any sequence of points and labels is unrealizable.

The labelling game above has been described for a fixed subset $S \subseteq \mathcal{X}$.
To extend the labelling game to generalize over the instance space $\mathcal{X}$, we allow the adversary to pick any finite subset $S \subseteq \mathcal{X}$ as an initialization step before the labelling game proceeds as usual.
By incorporating this simple modification, the labelling game can be easily extended to calculate $M_{sd}(H)$ where $M_{sd}(H) = \max_{S' \subseteq \mathcal{X}, |S'| < \infty} M_{sd}(H, S')$.

\subsection{Self-Directed Trees}
While for binary and multi-class classification $SDdim(H)$ can be defined in terms of a min-max analysis of the labelling game, we can alternatively formulate $SDdim(H)$ in terms of self-directed trees that aid in proving the upper and lower bounds for Theorem \ref{thm:SDdim=m_sd}.

\begin{definition}[Self-Directed Tree]
A \textit{self-directed tree} $T$ is defined as a tuple $(\mathcal{V}, E, O_{root})$ on $(H, S)$ where $\mathcal{V}$ is the collection of nodes, $E$ is a collection of edges, and $O_{root}$ is a set pertaining to the root node. Let $H$ be the concept class, $S$ a finite subset of the instance space $\mathcal{X}$, and $\mathcal{Y}$ the label set.

\begin{itemize}
    \item The root node is represented by $V_1 \in \mathcal{V}$ and $V_1 \subseteq S$. Then, $O_{root} \subseteq (S \setminus V_1 ) \times \mathcal{Y}$ where $\{ x':\exists y' \in \mathcal{Y}, (x', y') \in O_{root} \} = S \setminus V_1$.
    \item $\forall V \in \mathcal{V}$, if $V' = \mathrm{Parent}(V)$, then $V \subseteq V'$.
    \item $\forall e \in E$ where $e=(V', V)$ with $V' = \mathrm{Parent}(V)$, the edge weight is defined as $\omega(e) = ((x,y), O)$ where $x \in V'$ and $y \in \mathcal{Y}$. The set $O \subseteq V' \times \mathcal{Y}$ where $\{ x \} \cup \{ x':\exists y' \in \mathcal{Y}, (x', y') \in O \} = V' \setminus V$.
    \item Every vertex $V \in \mathcal{V}$ contains $2 \cdot |V|$ outgoing edges. For each $x \in V$, there exist edges $e, e'$ containing the property that $\omega(e) = ((x, y_x), O_{e})$ and $\omega(e') = ((x, y'_x), O_{e'})$ with $y_x, y'_x \in \mathcal{Y}, y_x \neq y'_x$.
    \item A branch $(V_1, ..., V_n)$ of length $n$, a root-to-leaf path, is realized by some function $h \in H$ if the following two conditions are met:
    \begin{itemize}
        \item For each pair $(x, y) \in O_{root}$, $h(x) = y$.
        \item For each $i$ such that $1 \leq i < n$, denote edge $e_i = (V_i, V_{i+1})$ with $\omega(e_i) = ((x_i, y_i), O_{e_i})$. It holds that $h(x_i) = y_i$ and $\forall~(x_m, y_m) \in O_{e_i}, h(x_m) = y_m$. 
    \end{itemize}
\end{itemize}

\end{definition}

A k-self-directed tree, or a \textit{k-tree}, is tree of this type such that the minimal length of any root-to-leaf path is $k$.

\begin{definition} 
    $\mathcal{T}^k(H, S)$ is the collection that represents all the $k$-trees $T$ built from $S$ such that every branch $b$ in $T$ is realized by a function $h_b \in H$.
\end{definition}

The design of the self-directed tree, specifically those in $\mathcal{T}^k(H, S)$, is constructed to mirror the game sequence of the labelling game on some $H$, $S \subseteq \mathcal{X}$, and $\mathcal{Y}$.
The set $O_{root}$ corresponds to Player A selecting a subset of points from $S$ to label before Player B's first move.
The edges in the self-directed tree correspond to Player B's selection of a point and label with the set $O$ corresponding to Player A's response of selecting a subset of points to label after Player B's move.
As a result, the notion of a single round corresponds to one edge in the mistake-tree.
In Appendix \ref{appendix:section_four}, we formally show the equivalency between the value of $SDdim(H, S)$ and the largest realizable $k$-tree on any $(H, S)$. 
\subsection{\textit{SDdim}}

The key interest in analyzing the labelling game is to gauge how many rounds the game will take on before terminating.
As mentioned previously, $SDdim$ was defined as the minimax value of the labelling game's payout where the payout is equal to $-1$ if the labelling of points results in an unrealizable sequence.
If one can build a $k$-tree on a non-empty $H$ and finite $S \subseteq \mathcal{X}$ with the additional property that all branches in the tree are realizable, then the labelling game on $(H, S)$ lasts at minimum $k$ rounds.
Therefore, we can alternatively formulate $SDdim$ as the depth of $k$-trees whose branches are realizable.

Now, we formally connect $SDdim$ to the payout of the game defined by these self-directed trees.
Given a $H$ and some finite $S \subseteq \mathcal{X}$ on which the labelling game is played upon, we can equivalently define
\begin{align}
    SDdim(H, S) = \max \{k \in \mathbb{N} \cup \{0\}: \mathcal{T}^k (H, S) \neq \emptyset \}
    \label{eq: 1}
\end{align}
\noindent as the largest $k$-tree one can build from $S$ on $H$. 
To observe the largest payout possible of the labelling game played on $H$, we define
\begin{align}
    SDdim(H) = \max_{S \subseteq \mathcal{X}, |S| < \infty} \max \{k \in \mathbb{N} \cup \{0\}: \mathcal{T}^k (H, S) \neq \emptyset \}.
    \label{eq: 2}
\end{align}
\noindent
We let $SDdim(\emptyset) = -1$ and for any non-empty $H$ and $S = \emptyset$, $SDdim(H, \emptyset) = 0$.

\subsubsection{Self-directed Standard Optimal Algorithm (SD-SOA)}

The SD-SOA is quite similar to the SOA designed by Littlestone \citep{littlestone:88}.
It keeps track of the current version space and point set which are denoted as $(H', S') \subseteq (H, S)$ respectively.

\begin{algorithm}
\caption{SD-SOA($H$, $S$, $\mathcal{Y}$)}\label{alg:cap}
\begin{algorithmic}[1]
\REQUIRE $H \neq \emptyset$
\REQUIRE $S \subseteq \mathcal{X}$
\STATE $H' = H$
\STATE $S' = S$
\FOR{$i=1,...,|S|$}
    \STATE Let $(x, y) = \argmin\limits_{(x', y') \in (S' \times \mathcal{Y})} \max\limits_{y'' \in \mathcal{Y} \setminus \{y'\}} SDdim(H'_{(x', y'')}, S' \setminus \{x'\})$
    \STATE Predict $(x, y)$
    \STATE Receive true label $y^*$
    \STATE Update $H' = H'_{(x, y*)}, S' = S' \setminus \{x\}$
\ENDFOR
\end{algorithmic}
\end{algorithm}

At each step, the SD-SOA selects the point with the best second-worst label and predicts in the direction of that point's worst label. 
The core property of this algorithm is that every time the SD-SOA's prediction is wrong, the value of $SDdim$ decreases at least by $1$.
This guarantee is provided by Lemma \ref{lemma:SDdim_point}, and the performance of this algorithm is analyzed in Lemma \ref{lemma:soa} showing that $M_{SD-SOA}(H, S) \leq SDdim(H, S)$.
\\

\section{Main Results} 

Theorem \ref{thm:SDdim=m_sd}, stated in the introduction, is a result that immediately follows from the stronger result proved in Theorem \ref{thm:sddim(H,S)}.
It's important to note that all the results below hold for $S \subseteq \mathcal{X}$ that is assumed to be finite.

\begin{theorem}
    For any $S \subseteq \mathcal{X}$ and concept class $H$, $SDdim(H, S)=M_{SD-SOA}(H, S)= M_{sd}(H, S)$.
    \label{thm:sddim(H,S)}
\end{theorem}

\begin{lemma}
For any $S \subseteq \mathcal{X}$, $SDdim(H, S) \leq M_{sd}(H, S)$.
\label{lemma:one}
\end{lemma}

\begin{proof}
A proof by induction will be established on the pair $(H, S)$ taking $S = \emptyset$ to be the base case.
In the base case, we show that $\forall H' \subseteq H$, $SDdim(H', \emptyset) = M_{sd}(H', \emptyset)$.
For any non-empty $H'$, $M_{sd}(H', \emptyset) = 0$, and the largest self-directed tree that can be constructed with $S = \emptyset$ has depth $0$ so $SDdim(H', \emptyset) = M_{sd}(H', \emptyset) = 0$.
If $H' = \emptyset$, then by definition $SDdim(\emptyset, \emptyset) = -1$ and $M_{sd}(\emptyset, \emptyset) = -1$.

Now, we apply the inductive step on $(H', S')$ when $H' \subseteq H$ and $S' \subset S$, then $SDdim(H', S') \leq M_{sd}(H', S')$.
The rest of the proof is now devoted to showing that $SDdim(H, S) \leq M_{sd}(H, S)$.

There must exist a $w \in S$ such that at most one $y_w' \in \mathcal{Y}$ has the largest value of $M_{sd}(H_{(w, y_{w'})}, S \setminus \{w \}) = M_{sd}(H, S)$.
If this were not true, then $\forall x \in S$, either $\exists y_x \in \mathcal{Y}$ where $M_{sd}(H_{(x, y_x)}, S \setminus \{x \}) > M_{sd}(H, S)$ or $\exists y_x, y'_x \in \mathcal{Y}$ where $M_{sd}(H_{(x, y_x)}, S \setminus \{x \}) = M_{sd}(H_{(x, y'_{x})}, S \setminus \{x \}) = M_{sd}(H, S)$.
If the learner chooses some $x \in S$ that fits the former scenario and any label $y \in \mathcal{Y}$, the adversary can always respond with $y_x$ as the true label with the guarantee that $M_{sd}(H_{(x, y_x)}, S \setminus \{x \}) > M_{sd}(H, S)$.
In the latter scenario, regardless of the learner's prediction $y \in \mathcal{Y}$ on $x \in S$, the adversary can always respond with a mistake saying that one of $y_x$ or $y'_x$ is the true label.
The version space on either $H_{(x, y_x)}$ or $H_{(x, y'_x)}$ has $M_{sd}(H_{(x, y_x)}, S \setminus \{x \}) = M_{sd}(H_{(x, y'_{x})}, S \setminus \{x \}) = M_{sd}(H, S)$.
Therefore, the learner will be forced to make at least $M_{sd}(H, S) + 1$ mistakes.

Let $T$ be the self-directed tree that realizes depth $SDdim(H, S)$.
In the tree $T$, $w$ can either be placed in $O_{root}$ or the root node.
If the point $w$ is placed in $O_{root}$ with label $y$, all branches of $T$ are realizable by the concepts in $H_{(w, y)}$.
Then $M_{sd}(H_{(w, y)}, S \setminus \{w \}) \leq M_{sd}(H, S)$, and using the inductive step, we establish that $SDdim(H, S) = SDdim(H_{(w, y)}, S \setminus \{w \}) \leq M_{sd}(H_{(w, y)}, S \setminus \{w \}) \leq M_{sd}(H, S)$.

If the point $w$ is placed in the root node, then it admits two distinct labels $y^1_w, y^2_w \in \mathcal{Y}$ from the two edges.
Then, the two corresponding subtrees will have depth at least $SDdim(H, S) - 1$. 
Additionally, note that each of the branches in these subtrees is realized by their respective version space $H_{(w, y^i_w)}$ for $i \in \{1, 2\}$.
It follows that 
$\forall i \in \{1, 2\}, SDdim(H_{(w, y^i_w)}, S \setminus \{w \}) \geq SDdim(H, S) - 1$.
At least one of the two labels cannot be $y'_w$ so without loss of generality, letting $y^2_w \neq y'_w$, we apply the inductive step to achieve $SDdim(H, S) \leq SDdim(H_{(w,y^2_w)}, S \setminus \{w \}) + 1 \leq M_{sd}(H_{(w,y^2_w)}, S \setminus \{w \}) + 1 \leq  M_{sd}(H, S)$.
\end{proof}

\noindent Below, Lemmas $\ref{lemma:monotonicity}$ and $\ref{lemma:SDdim_point}$ are used to describe additional properties of $SDdim$ that are useful in proving Lemma $\ref{lemma:soa}$.
The details about Lemmas $\ref{lemma:monotonicity}$ and $\ref{lemma:SDdim_point}$ are located in Appendix \ref{appendix:section_four}.

\begin{lemma}
    \label{lemma:monotonicity}
    Let $H$ and $S$ be any concept class and point set. 
    For any $n \in \mathbb{N}$ and \\ $\forall \sigma = \{(x_1, y_1), ..., (x_n, y_n) \} \in (S \times \mathcal{Y})^{n}$, $SDdim(H_{\sigma}, S \setminus \{x_1, ..., x_n \}) \leq SDdim(H, S)$. 
\end{lemma}

\begin{lemma}
    \label{lemma:SDdim_point}
    Let $H$ and $S \subseteq \mathcal{X}$ be a non-empty concept class and point set respectively. 
    There exists a point $x \in S$ such that at most one label $y_x \in \mathcal{Y}$ may have the property that $SDdim(H_{(x, y_x)}, S \setminus \{x \}) \geq SDdim(H, S)$.
\end{lemma}

\begin{lemma}
\label{lemma:soa}
Let $H$ be some non-empty concept class defined on the instance space $\mathcal{X}$ and a non-empty $S \subseteq \mathcal{X}$. 
Then, $M_{SD-SOA}(H, S) \leq SDdim(H, S)$.
\end{lemma}

\begin{proof}
A proof by induction will be established on the pair $(H, S)$ taking $S = \emptyset$ to be the base case.
In the base case, we show that $\forall H' \subseteq H$, $M_{SD-SOA}(H', \emptyset) = SDdim(H', \emptyset)$.
For any non-empty $H'$, $SDdim(H', \emptyset) = 0$, and since $S=\emptyset$, there are no points for the SD-SOA to query so $M_{SD-SOA}(H', \emptyset) = 0$.
If $H' = \emptyset$, then $SDdim(\emptyset, \emptyset) = -1$ and the mistake-bound of the SD-SOA is defined such that $M_{SD-SOA}(\emptyset, \emptyset) = -1$.

In the inductive step, we assume that for any $(H', S')$ where $H' \subseteq H$ and $S' \subset S$, the $M_{SD-SOA}(H', S') \leq SDdim(H', S')$.
The rest of the proof is now devoted to showing that $M_{SD-SOA}(H, S) \leq SDdim(H, S)$.

Let $x$ be the first point predicted by the SD-SOA and let $y_x$ be the label it predicts.
According to Algorithm $\ref{alg:cap}$, the SD-SOA algorithm, it picks a point whose second-worst label has the smallest value of $SDdim$.
The combination of Lemmas $\ref{lemma:monotonicity}$ and $\ref{lemma:SDdim_point}$ prove the existence of a point $w \in S$ such that at most one $y \in \mathcal{Y}$ has $SDdim(H_{(w, y)}, S \setminus \{w \}) = SDdim(H, S)$ and every other $y' \neq y$ has $SDdim(H_{(w,y')}, S \setminus \{w\}) < SDdim(H,S)$.
As a result, point $x$ satisfies this property as well (otherwise this $w$ would have been preferred over $x$ by the SD-SOA, by definition of the SD-SOA).
If the SD-SOA's prediction is correct, then the inductive hypothesis shows that $M_{SD-SOA}(H_{(x, y_x)}, S \setminus \{x \}) \leq SDdim(H_{(x, y_x)}, S \setminus \{x \}) \leq SDdim(H, S)$ (by Lemma~\ref{lemma:monotonicity}).
If the SD-SOA makes a mistake, then it is guaranteed that the value of $SDdim(H_{(x,y'_x)}, S \setminus \{x\}) \leq SDdim(H,S) - 1$ regardless of the true label $y'_x \neq y_x$ (by the property of $x$ established above).
By the inductive hypothesis, $M_{SD-SOA}(H_{(x, y'_x)}, S \setminus \{x \}) \!+\! 1 \leq SDdim(H_{(x, y'_x)}, S \setminus \{x \}) \!+\! 1 \leq SDdim(H, S)$.
\end{proof}

\begin{proof}[Proof of Theorem \ref{thm:sddim(H,S)}]
Lemma \ref{lemma:one} shows us that $SDdim(H, S) \leq M_{sd}(H, S)$.
From Lemma \ref{lemma:soa}, we know that for any $S \subseteq \mathcal{X}$, $M_{SD-SOA}(H, S) \leq SDdim(H, S)$.
For any $S \subseteq \mathcal{X}$, $M_{sd}(H, S)$ is the optimal self-directed learning mistake-bound, so $M_{sd}(H, S) \leq M_{SD-SOA}(H, S)$.
Therefore, $SDdim(H, S) \leq M_{sd}(H, S) \leq M_{SD-SOA}(H, S) \leq SDdim(H, S)$, so $SDdim(H, S) = M_{SD-SOA}(H, S) = M_{sd}(H, S)$.
\end{proof}
\section{Examples and Learnability Gaps}
In this section, we motivate the study of the self-directed complexity by applying it to popular examples found throughout the learning theory literature. 
In each example, we showcase the ease of calculating the self-directed complexity using the dimension $SDdim$ to characterize $M_{sd}$.

\paragraph{Example 1: Singleton Classifiers} Let $\mathcal{X} = \mathbb{N}$, the set of all natural numbers, $H = \{\mathbbm{1}_{\{n\}}: n \in \mathbb{N} \}$, and $\mathcal{Y}= \{0, 1\}$. $\mathbbm{1}_{\{n\}}$ is an indicator function on $n$.
To compute $M_{sd}(H)$, we play the labelling game on $H$.
From the subset of points selected by the adversary, the learner can label any point a $1$.
The adversary is then either forced to label all of the remaining points $0$ or risk leaving a point for the learner to label leading to an unrealizable sequence. 
Therefore, $M_{sd}(H) = SDdim(H) = 1$.

\paragraph{Example 2: Threshold Classifiers} \label{example2}Let the instance space $\mathcal{X} = [ 0, \infty)$ be on the real line, $H = \{ h_{a}: a \in [ 0, \infty) \}$, and $\mathcal{Y} = \{0,1\}$.
Each $h_a$ is a threshold function defined as $\mathbbm{1}_{[x < a]}$ which is $1$ if $x < a$ and $0$ elsewhere \citep{shalev2014understanding}. 
We now proceed to play the labelling game on $H$ to compute $M_{sd}(H)$.
In the beginning, the adversary is allowed to select any subset of points $S \subseteq \mathcal{X}$ to be given to the learner.
The optimal move for the learner would be to pick the leftmost point available, referred to as $x_{left}$, and label it $0$.
The adversary is now forced to label all the remaining points $0$ to maintain a realizable sequence.
If the adversary didn't label all the remaining points, the learner can label any remaining point a $1$ leading to an unrealizable sequence. 
As a result, the labelling game lasts a single round so $M_{sd}(H) = SDdim(H) = 1$.

\paragraph{Example 3: $k-$Interval Classifiers} Let $H_k$ be the class of $k-$interval classifiers over $\mathcal{X} = [0, \infty)$ with $\mathcal{Y} = \{ 0, 1\}$.
We show that $M_{sd}(H_k) = 2k$ by playing the labelling game and computing the value of $SDdim(H_k)=2k$.
Refer to the Appendix \ref{appendix:learnability_gaps} for a detailed description for the proof of the mistake-bound.

\subsection{Learnability Gaps from $M_{sd}$}

Before we dive into the learnability gaps, we focus on defining the other online learning mistake-bound model, $M_{online}$, and the two offline mistake-bound models $M_{worst}$ and $M_{best}$ on some concept class $H$ and instance space $\mathcal{X}$.
In the online learning scenario, the learner simply receives instances to make predictions on, so $M_{online}(H)$ is the mistake-bound model specific to this learning environment.
In the offline learning model $M_{worst}(H)$, called the worst-sequence model, the adversary chooses a sequence of instances that a learner sees before the prediction process starts.
It's counterpart, $M_{best}(H)$ which is the best-sequence model, the learner is allowed to choose a sequence of instances (without looking at the true labels) it wants to predict on.
We also look at the relationship between $VC(H)$ and $M_{sd}(H)$.
If $S$ is the largest shattered set on $H$, then the labelling game continues for a minimum of $VC(H)$ rounds since there exists a function $h \in H$ consistent with any labelling of $S$.
Then, it follows that $M_{sd}(H) \geq VC(H)$.
As a natural consequence, we get the following ordering of the mistake-bounds and VC-dimension for any $H$ and $\mathcal{X}$:
\begin{align}
    M_{online}(H) \geq M_{worst}(H) \geq M_{best}(H) \geq M_{sd}(H) \geq VC(H).
    \label{eq:all_mistake_bounds}
\end{align}
For a detailed technical description of these mistake-bounds, please refer to Appendix \ref{appendix:learnability_gaps}.

\subsubsection{$M_{best}$ vs $M_{sd}$} 

In this section, we explore the learnability gap between $M_{best}$ and $M_{sd}$, and also showcase a simple yet interesting result when $M_{best}(H) = 1$. 
Corollary \ref{corollary:m_best_vs_m_sd} is built from Theorem \ref{thm:m_best_learnability} to showcase the learnability gap between the two learning models.
For further technical details of the proofs, refer to Appendix \ref{appendix:m_best_proofs}.

\begin{theorem}
    For every $H$, instance space $\mathcal{X}$, and $\mathcal{Y}$, $M_{sd}(H) = 1$ if and only if $M_{best}(H) = 1$.
    \label{thm:m_best}
\end{theorem}

\begin{theorem}
For an $n \geq 3$, let the instance space 
$\mathcal{X}_n = \{x_1, x_2, ..., x_{2^n}\}$
where each $x_i \in \mathcal{X}_n$ is some distinct element. 
Then, for each $n \geq 3$, there exists a concept class $H_n$ with a finite label space $\mathcal{Y}_n$ such that 
$M_{best}(H_n) \geq n$ and $M_{sd}(H_n) = 2$.
\label{thm:m_best_learnability}
\end{theorem}

\begin{corollary}
    \label{corollary:m_best_vs_m_sd}
    There exists an instance space $\mathcal{X}$, label space $\mathcal{Y}$, and concept class $\mathcal{H}$ such that $M_{best}(\mathcal{H}) = \infty$ and $M_{sd}(\mathcal{H}) = 2$.
\end{corollary}

\noindent \textbf{Open Problem 1.} For some fixed $d \in \mathbb{N}$ and $\forall n \in \mathbb{N}$, does there exist a concept class $H^n_d$ over some instance space $\mathcal{X}$ in the binary class setting, $\mathcal{Y} = \{0,1\}$, such that $M_{sd}(H) = d$ and $M_{best}(H) = n + d$?
\\

\subsubsection{Other Learnability Gaps}

In this section, we briefly summarize each of the learnability gaps when compared to the self-directed learning complexity. 
Refer to Appendix \ref{appendix:learnability_gaps} for more detailed derivations.

\begin{itemize}
    \item \textbf{$\mathbf{M_{online}}$ vs. $\mathbf{M_{sd}}$} To describe this learnability gap, we use the concept class of threshold classifiers $H$ on $\mathcal{X}=[0,a)$ where $a \in \mathbb{R}_{>0}$. On $H$, we show that $M_{online}(H) = \infty$ and from Example 2, it is known that $M_{sd}(H)=2$.
    \item \textbf{$\mathbf{M_{worst}}$ vs. $\mathbf{M_{sd}}$} \cite{ben1997online} proved a result showing that $M_{worst} = \Omega(\sqrt{\log M_{online}})$ so on the concept class of threshold classifiers $H$, we show that $M_{worst}(H) = \infty$ while $M_{sd}(H) = 2$.
    \item \textbf{Query Learning vs. $\mathbf{M_{sd}}$} We explore the learnability gap between query learning and self-directed learning due to their inherent similarities.
    On the concept class of singletons $H$ where $\mathcal{X} = \mathbb{N}$, we show that $QC_{MQ}(H) = \infty$ while $M_{sd}(H) = 1$.
    \item \textbf{$\mathbf{VC}$ vs $\mathbf{M_{sd}}$}     The Vapnik-Chervonenkis dimension, $VC$, is a fundamental dimension within learning theory that characterizes concept classes that are PAC-learnable \citep{vapnik:71,vapnik:74}.
    \cite{ben1995self} demonstrated a learnability gap between $VC$ dimension and $M_{sd}$ by constructing a concept class $H^d_n$ such that for any $d \geq 3$ and $n \geq 1$, $VC(H^d_n) = d$ and $M_{sd}(H^d_n) = n + d$.
    We restate the example as given in \cite{ben1995self} and provide a detailed explanation in Appendix \ref{appendix:learnability_gaps}.
    \item \textbf{$\mathbf{TD}$ vs $\mathbf{M_{sd}}$} $TD$ refers to the minimum teaching dimension which is defined as the size of the worst minimum teaching set over some $H$ (\cite{hegedus:95,hanneke:07a}). To describe the learnability gap, we use the concept class $H^d_n$ constructed by \cite{ben-david:95} to show that $TD(H^d_n) = d + 1$.
\end{itemize}

\subsection{Families of Concept Classes}

\subsubsection{Axis-Aligned Rectangles in $\mathbf{\mathbb{R}^d}$} 
\label{section:rectangles}
Let the concept class $H_d$ represent the class of axis-aligned rectangles and the instance space $\mathcal{X} = \mathbb{R}^d$ for some $d \in \mathbb{N}$. We formally define $H_d$ as $H_d = \{h_{(c_{1_l}, c_{1_r}),..., (c_{k_l}, c_{k_r})}: c_{1_l} \leq c_{1_r}, ..., c_{k_l} \leq c_{k_r}  \}$ \citep{shalev2014understanding}.
Each classifier $h_{(c_{1_l}, c_{1_r}), ..., (c_{k_l}, c_{k_r})}$ is then defined in the following way:

\[ h_{(c_{1_l}, c_{1_r}), ..., (c_{k_l}, c_{k_r})}(x_1, ..., x_k) = \begin{cases} 
          1 & \text{if}~ c_{1_l} \leq x_1 \leq c_{1_r}, ..., c_{k_l} \leq x_k \leq c_{k_r} \\ 
          0 & \text{otherwise.} \\ 
       \end{cases}
    \]

\citet{goldman1994power} explored axis-aligned rectangles on a specific subset $S \subseteq \mathcal{X}$ that was equivalent to a mesh-grid.
Our analysis is the first known result that computes $M_{sd}(H_d)$.

\begin{theorem}
    For any $d \in \mathbb{N}$ and letting $H_d$ be the concept class of axis-aligned rectangles on $\mathcal{X} = \mathbb{R}^d$, $M_{sd}(H_d) = 2d$.
\end{theorem}

\begin{proof}
To calculate $M_{sd}(H_d)$, we turn to the labelling game and design an optimal strategy for the learner.
Let $S \subseteq \mathcal{X}$ be any subset of points selected by the adversary.
Without loss of generality, the learner will select the $d^{th}$ dimension and pick the leftmost point in that dimension and label it a $1$.
After the adversary's turn, the learner will proceed to pick the rightmost point in the $d^{th}$ dimension and label it a $1$.
The learner then follows this strategy for the remaining $d - 1$ dimensions which would imply labelling $2d - 2$ additional points.
The learner purposefully targeted the extreme points in every dimension to create the largest possible rectangle in $\mathbb{R}^d$ containing any unlabelled points.
So, any unlabelled point lying within the axis-aligned rectangle constructed on those $2d$ points must be labelled as a $1$.
If the adversary proceeds to label all the remaining unlabelled points, then the game ends in $2d$ rounds.
If the adversary intentionally leaves even a single point unlabelled, then the learner will label that point a $0$ immediately producing an unrealizable sequence.
Therefore, $SDdim(H_d) \leq 2d$.
Since $VC(H_d) = 2d$ and $VC(H_d) \leq SDdim(H_d)$, then $M_{sd}(H_d) = SDdim(H_d) = VC(H_d) = 2d$. 
\end{proof}

\subsubsection{VC-Dimension $1$ Classes} We now extend the study of self-directed learning to the entire family of concept classes that contain $VC(H) = 1$. 
This family of concept classes has been originally studied by \citet{ben-david:15}, and it proves the existence of non-trivial structures that characterize all concept classes that contain $VC(H) \leq 1$.
We focus on $VC(H)=1$ concept classes with the proof of Theorem $\ref{thm:vc_one}$ located in Appendix \ref{appendix:vc_one}.

\begin{theorem}
    For any concept class $H$, $M_{sd}(H) = 1$ if and only if $VC(H) = 1$.
    \label{thm:vc_one}
\end{theorem}

\subsubsection{Linear Separators}

In this section, we consider the class of linear separators, and we prove a lower bound on the self-directed complexity on this family of functions.
Let $\mathcal{X} = \mathbb{R}^d$, $\mathcal{Y} = \{0, 1\}$, and $\mathcal{H}_d = \{x \mapsto sign(h_{\mathbf{w},b}(x)): h_{\mathbf{w},b}, \mathbf{w} \in \mathbb{R}^d, b \in \mathbb{R} \}$.
The function $h_{\mathbf{w},b}(x) = \left( \sum_{i=1}^d w_i x_i \right) + b$ represents an affine function parameterized by $\mathbf{w} \in \mathbb{R}^d$ and $b \in \mathbb{R}$.
It is widely known that the VC-dimension of a class of linear separators on $\mathbb{R}^d$, $\mathcal{H}_d$, has $VC(\mathcal{H}_d) = d + 1$ (\cite{shalev2014understanding}).

Below, we state two theorems in showing a lower bound of $M_{sd}(\mathcal{H}_d) = SDdim(\mathcal{H}_d) \geq 4 \lfloor \frac{d}{3} \rfloor$.
Our first theorem presents an example of a point configuration in $2$ dimensions (diagram in Fig. \ref{fig:linear-separators} and Fig. \ref{fig:lower-bound-configurations}) such that $SDdim(\mathcal{H}_{2}) \geq 4$.
Then, our second theorem uses this result and extrapolates it for any arbitrary number of dimensions to show that $SDdim(\mathcal{H}_d) \geq 4 \lfloor \frac{d}{3} \rfloor$.

\begin{theorem}
    In $\mathbb{R}^2$, $M_{sd}(\mathcal{H}_{2}) \geq 4$. 
    \label{thm:two-d-linear-separators}
\end{theorem}

\begin{proof}(\textbf{Sketch})

To show a lower bound of $M_{sd}(\mathcal{H}_2) \geq 4$, we show that $SDdim(\mathcal{H}_2) \geq 4$ by showing that there is a strategy for the adversary to extend the labelling game for a minimum of $4$ rounds.
Let a subset $S \subset \mathbb{R}^2$ contain $8$ points as the configuration of Fig. \ref{fig:linear-separators}(a) depicts.
In the first two rounds of the labelling game, the adversary will not label any points on its turn.
This implies that after two rounds of the game, only two points are labelled and those labels are selected by the learner. 
Then, the adversary will select labels for an appropriate set of points such that it matches one of the four configurations in Fig. \ref{fig:lower-bound-configurations} up to rotation, reflection, and/or taking complements of the labels (interchanging $1$s and $0$s).
From here, we show that the labelling game is guaranteed to extend for an additional $2$ rounds for a total of $4$ rounds.

Now, it's important to note that the learner has ${8 \choose 2}$ or $28$ different pairs of points it can label in the first two rounds of the game.
However, the learner can also label any pair of points in $4$ possible ways: $(0, 0)$, $(0, 1)$, $(1, 0)$, and $(1, 1)$.
So, there are $112$ distinct labelling combinations of any $2$ points from this $8$ point configuration.

\begin{figure}[h]
    \centering
    \begin{subfigure}
        \centering
        \includegraphics[width=0.4\textwidth]{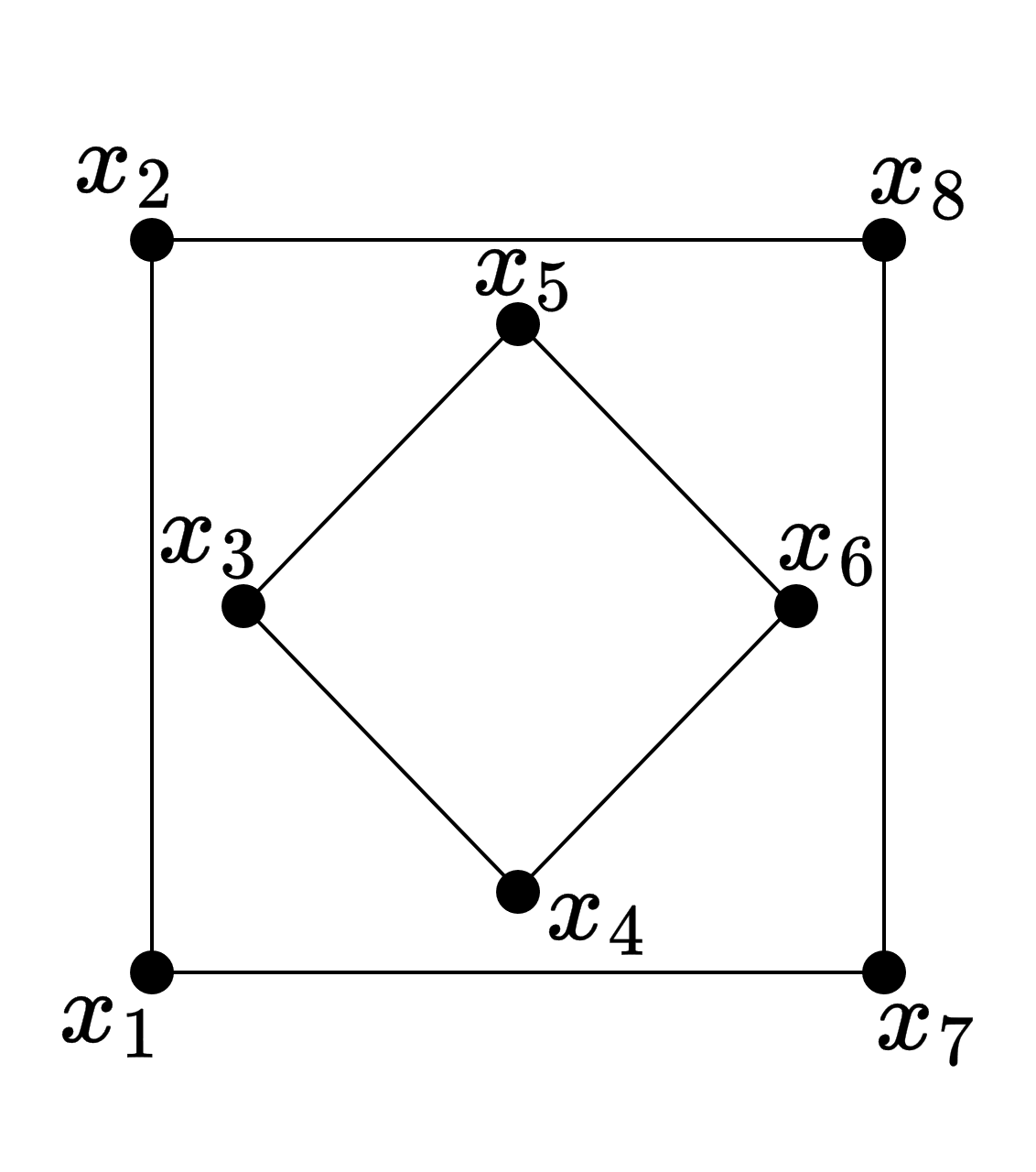}
        \label{base-diagram}
    \end{subfigure}
    \begin{subfigure}
        \centering
        \includegraphics[width=0.4\textwidth]{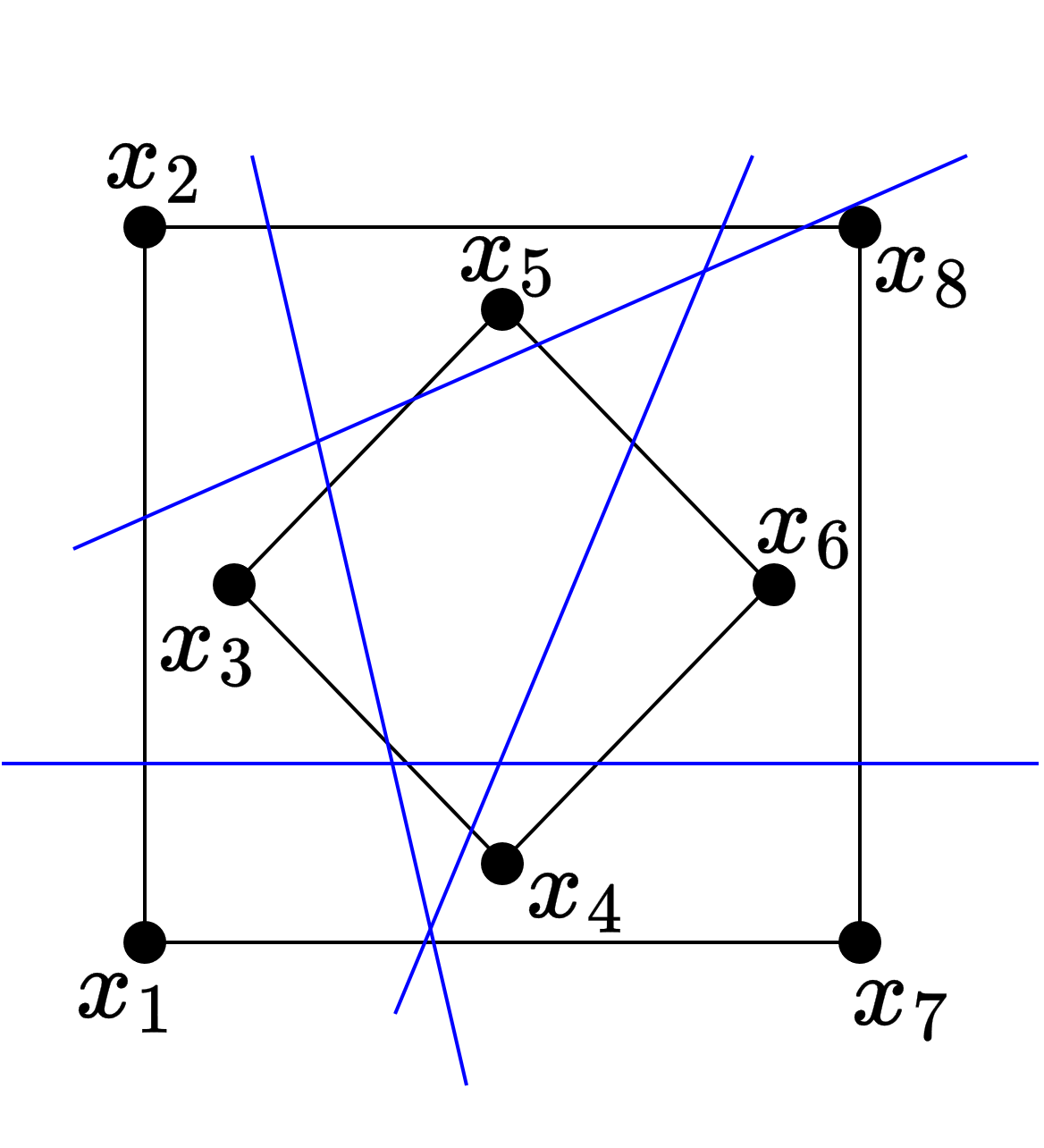}
    \end{subfigure}
    \caption{The 2D configuration of points that yields a lower bound of $SDdim(\mathcal{H}_2) \geq 4$. (a) represents the configuration of $8$ points and (b) represents some examples of linear separators on this set of points denoted as solid, blue lines.}
    \label{fig:linear-separators}
\end{figure}

To show that the labelling game can extend for a minimum of $4$ rounds, the adversary must be equipped with a strategy that regardless of how the learner labels the first $2$ points, there exists a configuration such that the labelling game can continue for an additional $2$ rounds.
Fig. \ref{fig:lower-bound-configurations} represents four base configurations the adversary can employ to push the game for an additional $2$ rounds.
If the labelling game were played on any on the four configurations in Fig. \ref{fig:lower-bound-configurations} as a starting point, then it will last for $2$ rounds.
This can be easily verified by noting that for any configuration that has $2$ unlabelled points, any labelling of those points yields a linearly separable classification.
In the case of configurations that have $3$ or more unlabelled points, the learner is allowed to label any one point on its third move.
Depending on the point and label selected by the learner, the adversary will label a particular sequence of points appropriately and leave one point unlabelled for the learner.
On the learner's fourth move, regardless of the label of the last unlabelled point, a linearly separable classification always exists.

It is important to note that this observation also holds under any rotation, reflection, or complements of the labels.
For any set of linear separators consistent on the base configuration, applying the same modifications on each consistent linear separator as applied to the configuration will still yield a set of consistent linear separators.

For example, assume that the learner has labelled the following two points after two rounds of the labelling game: $(x_5, 1)$ and $(x_7, 0)$.
Due to the symmetric nature of the $8$-point configuration, the adversary can choose to take the complement of configuration (b) in Fig. \ref{fig:lower-bound-configurations} and flip the entire configuration across the y-axis to get $x_5$ labelled as $0$ and $x_7$ labelled as $1$.
Then, the adversary can label the necessary points accordingly, and the labelling game will have continued for $4$ rounds.

The adversary's strategy allows for the labelling game to continue for $4$ rounds regardless of the learner's selection of points and labels; as a result, $M_{sd}(\mathcal{H}_2) = SDdim(\mathcal{H}_2) \geq 4$. 

\begin{figure}
     \centering
     \begin{subfigure}
         \centering
         \includegraphics[width=0.40\textwidth]{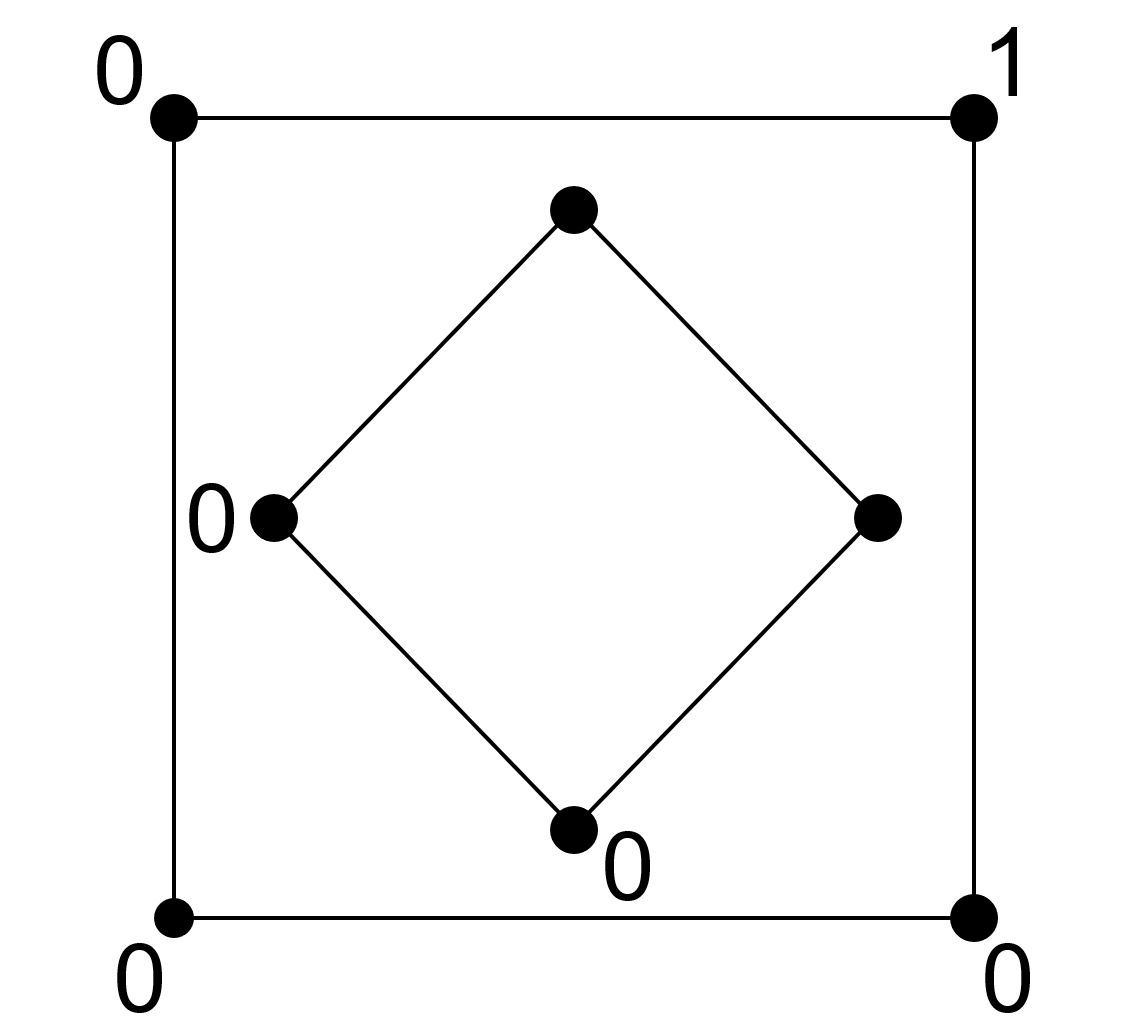}
         \label{fig:(a)}
     \end{subfigure}
     \begin{subfigure}
         \centering
         \includegraphics[width=0.41\textwidth]{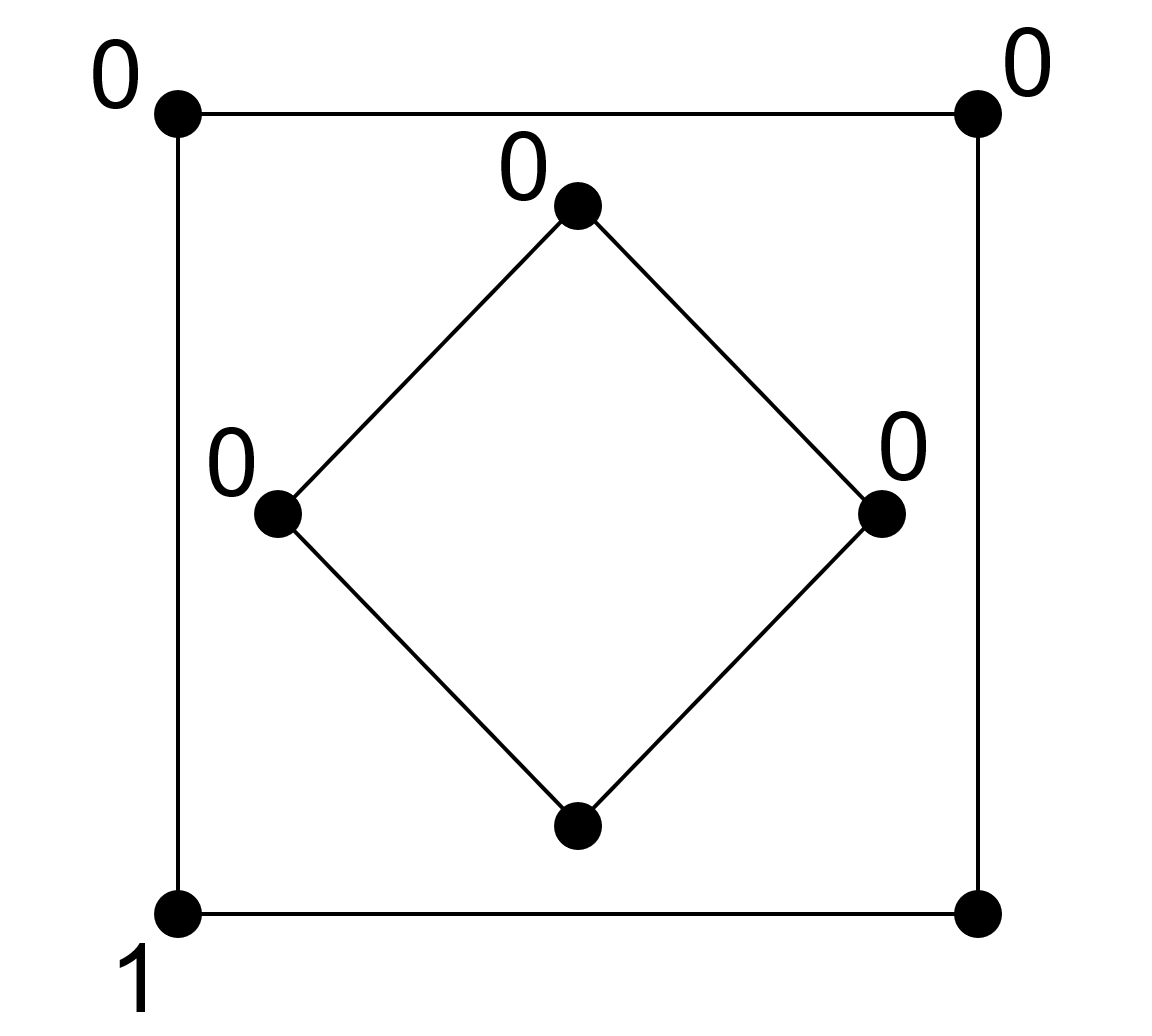}
         \label{fig:(b)}
     \end{subfigure}
     \begin{subfigure}
         \centering
         \includegraphics[width=0.40\textwidth]{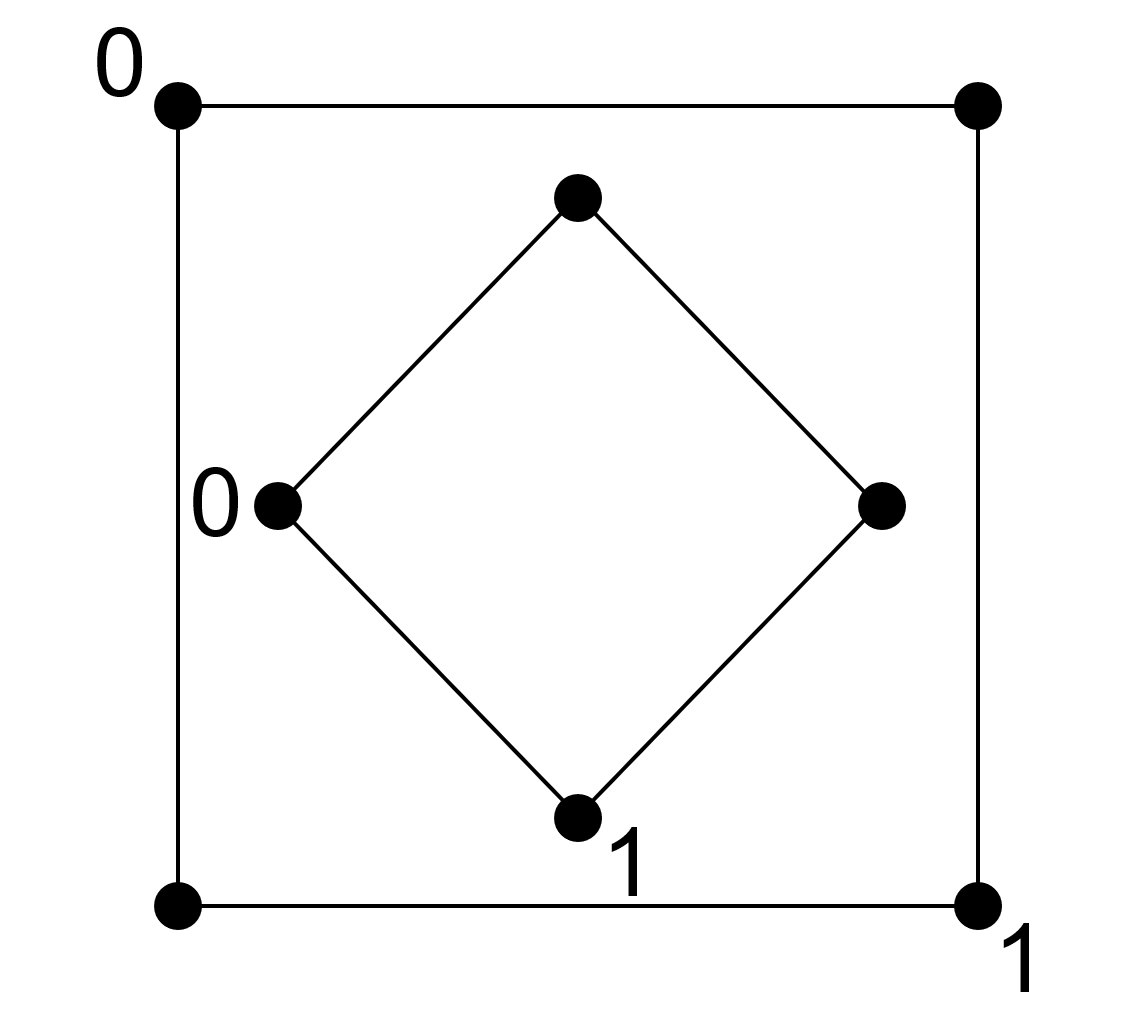}
         \label{fig:(c)}
     \end{subfigure}
    \begin{subfigure}
        \centering
        \includegraphics[width=0.37\textwidth]{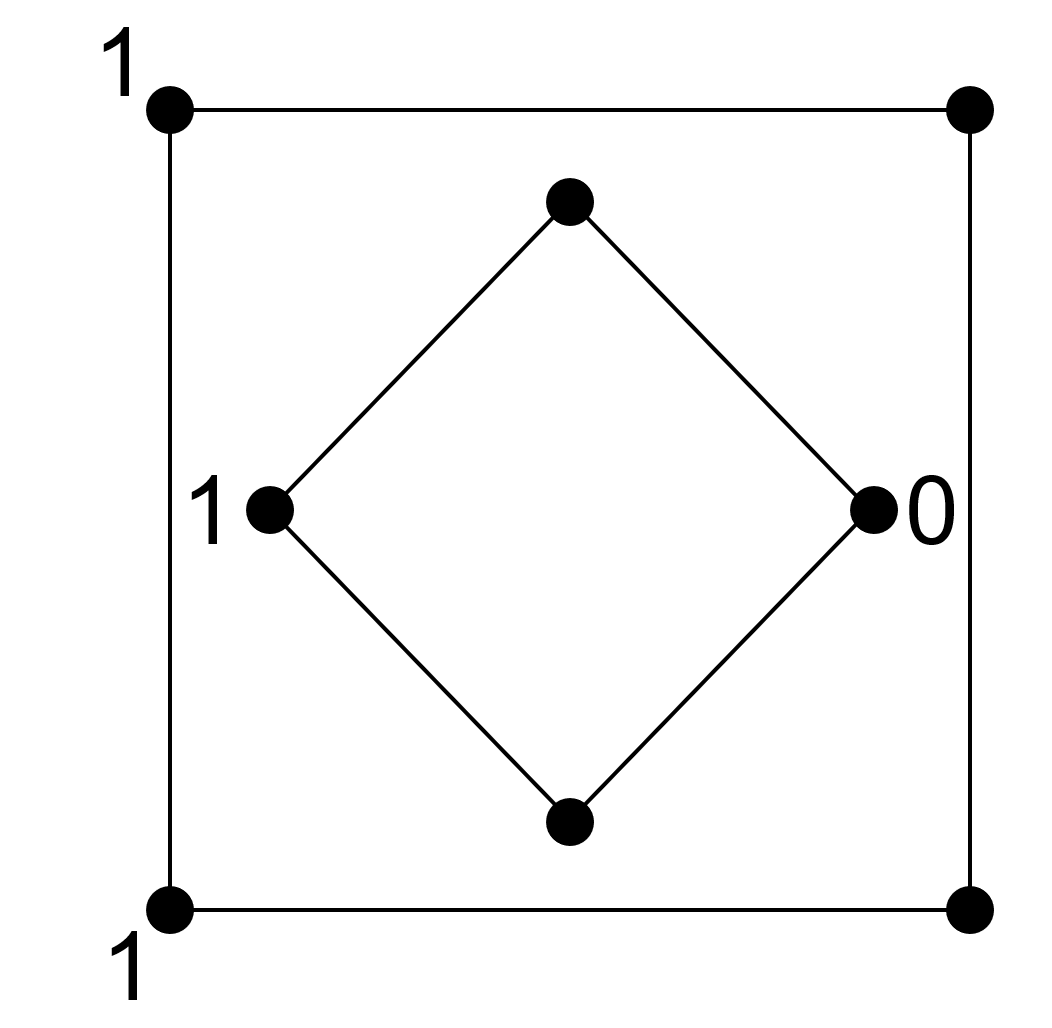}
        \label{fig:(d)}
    \end{subfigure}
     
    \caption{The set of configurations from which the adversary can choose to implement to guarantee $4$ rounds in the labelling game. After two rounds of the labelling game where only the learner labels points, the adversary can label any set of points to fulfill one of the four configurations to guarantee $4$ rounds in the labelling game.}
    \label{fig:lower-bound-configurations}
\end{figure}

\end{proof}

\begin{theorem}
    In $\mathbb{R}^d$ for $d \geq 3$, $M_{sd}(\mathcal{H}_d) \geq 4 \lfloor \frac{d}{3} \rfloor$.
    \label{thm:R_d_linear_separators}
\end{theorem}

\begin{proof}
    Let $l = \lfloor \frac{d}{3} \rfloor$.
    Select a group of three distinct dimensions from $\mathbb{R}^d$ and construct the dataset represented in Fig. \ref{fig:linear-separators}(a) on two of those three dimensions while fixing the third to be $1$ and all the other dimensions set to $0$.
    Repeat this procedure $l - 1$ times on disjoint sets of three distinct dimensions.
    Essentially, the two-dimensional configuration is embedded multiple times within $\mathbb{R}^d$ on disjoint groups of three dimensions.
    From Theorem \ref{thm:two-d-linear-separators}, we know that any two-dimensional configuration of points has a lower bound of $M_{sd}(\mathcal{H}_2) \geq 4$.
    Since the datasets are disjoint, then each two-dimensional configuration of points can be treated as an independent labelling game. 
    Additionally, there exists a linear separator $h \in \mathcal{H}_d$ that is the composite of all the individual linear separators that solve their respective two-dimensional configurations so $M_{sd}(\mathcal{H}_d) = SDdim(\mathcal{H}_d) \geq 4 \lfloor \frac{d}{3} \rfloor$.
    
\end{proof}

\section{Agnostic Setting}
\label{section:agnostic}

Up till now, the focus of the paper has centered around $SDdim$ as the dimension that characterizes the mistake-bounds of realizable self-directed learning under the binary and multi-class settings. 
In this section, we now pivot our attention to understand self-directed learning in the agnostic setting under binary classification.
We present rather interesting results showing that one can compute upper and lower bounds on regret in agnostic self-directed learning without using $SDdim$.

We first define the agnostic setting for binary classification.
Given some concept class $H$, a finite multiset $S \subseteq \mathcal{X}$, $\mathcal{Y} = \{0, 1\}$, and the number of rounds $T = |S|$, the agnostic setting works as described below.
For each round $t=1,...,T:$

\begin{enumerate}
    \item The learner chooses some $x_t \in S$.
    \item Given the previous sequence of points and labels $(x_1, y_1), ..., (x_{t-1}, y_{t-1})$ and $x_t$, the learner makes a prediction $\hat{y}_t$.
    \item The environment determines the true label $y_t$ for $x_t$ without knowing $\hat{y}_t$ but it has access to the predictions of the previous rounds $(x_1, \hat{y}_1),...,(x_{t-1}, \hat{y}_{t-1})$.
    \item Update $S = S \setminus \{ x_t \}$.
\end{enumerate}

With the constraint of realizability removed, the target concept $f^*$ that is consistent on the sequence of points and true labels does not necessarily exist within $H$.
It's also important to notice a subtle distinction between the environment's role as compared to in the realizable setting.
The environment in the agnostic setting is not allowed to witness the prediction $\hat{y}$ before responding with the true label whereas the environment could base its true label from $\hat{y}$ in the realizable setting (the two scenarios are equivalent for deterministic learners, though not equivalent for randomized learners).
As is well known (\cite{cesa-bianchi:06,ben2009agnostic}), in the agnostic setting allowing the environment to base the label on the learner's predictions leads to non-vanishing regret because the constraint of realizability has been removed.
Additionally, we extend our analysis to allow the learner to randomize its predictions to prevent the environment from correctly ``guessing'' the learner's output.
We incorporate these modifications as part of our regret bounds in the agnostic self-directed setting.
For any concept class $H$, any self-directed learning algorithm $\mathcal{A}$, and a sequence of $T$ examples $x_1, x_2, ..., x_T$ chosen by $\mathcal{A}$ from the given multiset $S$, we define the regret, $R_{sd}(H, \mathcal{A})$, as 
\begin{align}
    R_{sd}(H, \mathcal{A}) = \max_{S \subseteq \mathcal{X}, |S| < \infty} \mathbb{E} \left[ \sum_{t=1}^T |\hat{y}_t - y_t| - \min_{h \in H} \sum_{t=1}^T |h(x_t) - y_t| \right].
    \label{eq:regret}
\end{align}
In Equation $\ref{eq:regret}$, $\hat{y}_t$ represents the prediction made by $\mathcal{A}$ on round $t$, $y_t$ is the true label on round $t$, and $S$ is a finite multiset.
The expectation $\mathbb{E}$ is taken over the labels $\hat{y}_t$ and $y_t$ that are dictated 
by the play between the learner and adversary.
\noindent In our analysis, we are concerned with the minimum expected regret over all algorithms:
\begin{align*}
    R_{sd}(H) = \min_{\mathcal{A}} R_{sd}(H, A) . 
\end{align*}

\noindent Specifically, we are interested in discovering the optimal expected regret bounds for any $H$.

In the adversarial online learning setting, \cite{ben2009agnostic} showed that the expected regret bound is $O\left(\sqrt{LD(H)T\mathrm{ln}(T)}\right)$ with a lower bound of $\Omega\left(\sqrt{LD(H)T}\right)$ where $LD(H)$ represents the Littlestone dimension of $H$ and $T$ is the number of rounds. \footnote{The result of \cite{alon:21} has shown that if $H$ satisfies a mild restriction (implying a certain induced game satisfies a minimax theorem), then the expected regret is $\Theta(\sqrt{LD(H)T})$.  It still remains open whether $\Theta(\sqrt{LD(H)T})$ is the optimal regret without any restrictions on the class $H$.}
It would seem natural to extend this result for the self-directed setting by upgrading the regret bounds with $SDdim(H)$, however, we obtain the stronger expected regret bound of $O\left(\sqrt{VC(H)T\mathrm{ln}(T)}\right)$ with a lower bound of $\Omega\left(\sqrt{VC(H)T}\right)$.

\subsection{Why $VC(H)$ Instead of $SDdim(H)$?}

At first sight, the natural approach would be to derive regret bounds based on $SDdim(H)$ to characterize regret bounds in the binary agnostic setting.
In the analysis of adversarial online learning in the binary agnostic setting, \citet{ben2009agnostic} bounded the number of experts based on the Littlestone dimension of the concept class.
In a similar fashion, one could build experts based on the SD-SOA and instantiate a set of experts dependent on $SDdim(H)$.
Would the optimal approach for agnostic self-directed learning follow the strategy used in agnostic adversarial online learning?

To answer this question, a quick revisit of agnostic adversarial online learning can illuminate the key difference.
In the adversarial setting, the full sequence of instances is unknown to the learner so each SOA-based expert simulates mistakes at different positions within the sequence. 
The SOA will make at most $LD(H)$ mistakes before it narrows down the version space consisting of the target concept(s). 
As a result, on $T$ rounds, there exists ${T \choose \leq LD(H)}$ different experts to simulate at most $LD(H)$ mistakes.

In the self-directed setting, the power exists to adaptively choose the sequence in real-time but the learner can also preselect any arbitrary sequence of instances.
This eliminates one of the biggest hurdles faced in the adversarial setting where both the points and their labels are unknown.
Then, one can assign an expert to each realizable classification of $H$ on a fixed sequence of instances.
Therefore, letting $\hat{T} = \{x_1, ..., x_T \}$ contain the $T$ selected points, the number of experts corresponds to the projection of $H$ on $\hat{T}$ which is bounded by ${T \choose \leq VC(H)}$. 

\subsection{Agnostic Self-Directed Learning Algorithm}

Below, we formally describe three algorithms we use for the agnostic setting. 
Our Algorithm \ref{alg:agnostic}, the agnostic self-directed learning algorithm, is inspired by the ideas from \cite{ben2009agnostic} that uses a learning with expert advice algorithm to make predictions.
Algorithm \ref{alg:weighted_majority} implements a multiplicative weights algorithm which is a form of learning with expert advice (\cite{littlestone:94, cesa-bianchi:97, cesa-bianchi:06})
Our last algorithm, Algorithm \ref{alg:expert}, describes the expert used within Algorithm \ref{alg:weighted_majority}.
In all our algorithms, $S$ is allowed to be a finite multiset that is a subset of $\mathcal{X}$ and the number of rounds $T$ is equivalent to $|S|$ so the self-directed learner can classify all the points.

\begin{algorithm}
\caption{Expert($x_t$)}
\label{alg:expert}
\begin{algorithmic}[1]
\REQUIRE Initialized with some concept $h \in H_S$
\STATE $y_{pred} = h(x_t)$
\RETURN $y_{pred}$
\end{algorithmic}
\end{algorithm}

\begin{algorithm}
\caption{Learning with Expert Advice($N$, $S$, learning rate $\eta$)}
\label{alg:weighted_majority}
\begin{algorithmic}[1]
\REQUIRE $\mathbf{w}^0 = (1, ..., 1) \in \mathbb{R}^N; Z_0 = N$, $N$ is the number of experts
\STATE $T = |S|$
\FOR{$t=1, 2, ..., T$}
    \STATE Let $x_t$ be the $t^{th}$ point from $S$
    \STATE Receive expert advice on $t^{th}$ round $(\mathrm{Expert}^t_1(x_t), ..., \mathrm{Expert}^t_N(x_t)) \in \{0, 1\}^N$
    \STATE True label $y_t$ determined but not given to learner
    \STATE Define $\hat{p}_t = \frac{1}{Z_{t-1}}\sum_{i:\mathrm{Expert}^t_i=1} w_i^{t-1}$
    \STATE Predict $\hat{y}_t = 1$ with probability $\hat{p}_t$
    \STATE Receive label $y_t$
    \STATE Update: $w^t_i = w^{t-1}_i \mathrm{exp}(-\eta |\mathrm{Expert}^t_i - y_t|); Z_t = \sum_{i=1}^N w^t_i$
\ENDFOR
\end{algorithmic}
\end{algorithm}

\begin{algorithm}
\caption{Agnostic Self-Directed Learning Algorithm($H$, $S$)}
\label{alg:agnostic}
\begin{algorithmic}[1]
\REQUIRE $H \neq \emptyset$
\STATE Let $H_S = \{(h(x_1), ..., h(x_{|S|})): h \in H \}$ be the projection of $H$ on $S$.
\FORALL{$h \in H_S$}
    \STATE Create an Expert initialized with $h$
\ENDFOR
\STATE Set $N = |H_S|$
\STATE Run Algorithm \ref{alg:weighted_majority} $(N, S, \eta)$
\end{algorithmic}
\end{algorithm}

\newpage
\subsection{Expected Regret Bound for Agnostic Self-Directed Learning}

To derive the expected regret bound, we will first draw upon a classic result in learning from expert advice \citep{cesa-bianchi:06} proving that the expected regret bound for Algorithm \ref{alg:weighted_majority} is at most $\sqrt{\frac{1}{2} \mathrm{ln}(N) T}$ when $\eta = \sqrt{8 \mathrm{ln}(N) /T}$.
Since Algorithm \ref{alg:agnostic} executes Algorithm \ref{alg:expert} on a set of experts that contain all possible realizable classifications on $T$ points, there exists an expert in Algorithm \ref{alg:expert} that makes at most as many mistakes as the best $h \in H$ makes. 
Combining the previous fact with expected regret bound for Algorithm \ref{alg:weighted_majority}, we obtain the following short theorem stating our result for the expected regret bound from  Algorithm \ref{alg:agnostic}.

\begin{theorem}
    For any concept class $H$, $\mathcal{X}$, and any sequence of $T$ examples, the expected regret in the agnostic self-directed case has 
    \begin{align*}
        R(H) = O\left(\sqrt{VC(H) \cdot T \mathrm{ln}(T)}\right).
    \end{align*}
\end{theorem}

\begin{proof}
    We run Algorithm \ref{alg:agnostic} on $(H, S)$ where $S \subseteq \mathcal{X}$ is a finite multiset and $T = |S|$.
    Now, the number of experts $N$ is directly tied to the size of $H_S$, the projection of $H$ onto the set $S$.
    Then, we draw upon Sauer's Lemma to show that $|H_S| \leq \sum_{i=0}^d {T \choose i} \leq \left( \frac{eT}{VC(H)} \right)^{VC(H)}$.
    Then, we analyze the expected regret bound of Algorithm \ref{alg:weighted_majority} with the number of experts $N = \left( \frac{eT}{VC(H)} \right)^{VC(H)}$:
    \begin{align*}
        R(H) &= \min_{\mathcal{A}} \max_{S \subseteq \mathcal{X}, |S| < \infty} \mathbb{E} \left[ \sum_{t=1}^T |\hat{y}_t - y_t| - \min_{h \in H} \sum_{t=1}^T |h(x_t) - y_t| \right]\\
        & \leq \sqrt{\frac{1}{2}\mathrm{ln}\left( \left( \frac{eT}{VC(H)} \right)^{VC(H)} \right)T} \\
        &\leq \sqrt{\frac{1}{2} VC(H) \cdot T~\mathrm{ln}(eT) } \\
        & = O\left(\sqrt{VC(H) \cdot T \mathrm{ln}(T)}\right)
    \end{align*}
\end{proof}

\subsection{Lower Bound}

In this section, we use the techniques in Lemma $14$ from \cite{ben2009agnostic} to prove that no algorithm can achieve a regret worse than $\Omega(\sqrt{VC(H)T})$.

\begin{theorem}
    For any concept class $H$ with $VC(H) = d < \infty$, there exists an instance space $\mathcal{X}_d$ such that any self-directed learning algorithm will have a minimum expected regret of 
    \begin{align*}
        R(H) \geq \sqrt{\frac{VC(H)T}{8}}.
    \end{align*}
\end{theorem}

\begin{proof}
    For any $k \in \mathbb{N}$, let $T=kd$. 
    Let $\{m_1, m_2, ..., m_d \}$ be a set of points shattered by $H$.
    Let the instance space $\mathcal{X}_d = \{z_1, ..., z_k, ..., z_{(d-1)k+1}, ..., z_{dk}\}$ where the first $k$ points are copies of $m_1$, the second set of $k$ points are copies of $m_2$, and so on and so forth.
    For each $z_i \in \mathcal{X}_d$, sample a label $y_i$ i.i.d. from a $Bernoulli(1/2)$ distribution.
    Now, we analyze the expected number of mistakes any algorithm will make in $T$ rounds on the instance space $\mathcal{X}_d$. 

    \begin{align*}
        \mathbb{E}\left[ \sum_{t=1}^T |\hat{y}_t - y_t|  \right] = \sum_{t=1}^T \mathbb{E}[|\hat{y}_t - y_t|]
        = \sum_{t=1}^T \frac{1}{2} 
        = \frac{T}{2}
    \end{align*}

    To analyze the expected number of mistakes the best $h \in H$ will make, we will split $\mathcal{X}_d$ into blocks $T_1$, $T_2$, ..., $T_d$ where each $T_i$ will contain the $k$ copies corresponding to $m_i$.
    We analyze a single block, $T_i$, and then extend the analysis to all other blocks. 
    Let $r_i = \sum_{t \in T_i} y_t$.
    For the first block, we have that 
    \begin{align*}
        \min_{y_i \in \{0,1 \}} \sum_{t \in T_i} |y_i - y_t| = \begin{cases} 
          k-r_i & r_i \geq k/2 \\
          r_i & r_i < k/2
       \end{cases}
    \end{align*}

    \noindent If we incorporate the expected number of mistakes the algorithm will make on the rounds corresponding to selecting point $m_i$, then 

    \begin{align*}
        \mathbb{E}\left[ \sum_{t \in T_i} |\hat{y}_t - y_t | \right] - \mathbb{E} \left[ \min_{y_i \in \{0,1 \}} \sum_{t \in T_i} |y_i - y_t| \right] = \frac{k}{2} - \mathbb{E} \left[ \min_{y_i \in \{0,1 \}} \sum_{t \in T_i} |y_i - y_t| \right] = E[|r_i - k/2|]
    \end{align*}

    To analyze $E[|r_i - k/2|]$, one can notice that since $r_i$ is the sum of $k$ i.i.d. variables that are either $0$ or $1$ with a constant factor of $k/2$ subtracted, this is equivalent to adding $k$ i.i.d. variables that are either $-1$ or $1$ and dividing by $2$.
    With this observation, one can use Khinchine's inequality to derive a lower bound for the remaining expectation.
    Let each $a_i = 1$ and each $\sigma_i$ be an i.i.d. sign variable with $P(\sigma_i = 1) = P(\sigma = -1) = 1/2$ (\cite{cesa-bianchi:06}).

    \begin{align*}
        E[|r_i - k/2|] = \frac{1}{2}\mathbb{E}\left[ \left| \sum_{i=1}^k a_i \sigma_i \right| \right] \geq \frac{1}{2} \cdot \frac{1}{\sqrt{2}} \sqrt{\sum_{i=1}^k a^2_i} = \sqrt{\frac{k}{8}}
    \end{align*}

    Now, we simply extend this for each of the $d$ blocks and get the following result: 
    \begin{align*}
        \mathbb{E}\left[ \sum_{t=1}^T |\hat{y}_t - y_t| - \min_{h \in H} \sum_{t=1}^T |h(x_t) - y_t| \right] &= \mathbb{E}\left[ \sum_{t=1}^T |\hat{y}_t - y_t|  \right] - \mathbb{E}\left[ \min_{h \in H} \sum_{t=1}^T |h(x_t) - y_t| \right] \\
        &= \frac{T}{2} -  \mathbb{E}\left[ \min_{h \in H} \sum_{t=1}^T |h(x_t) - y_t| \right] \\
        &\geq d\sqrt{\frac{k}{8}} = \sqrt{\frac{dT}{8}}
    \end{align*}

    Here, $x_t$ refers to the point selected by the algorithm on the $t^{th}$ timestep. 
    Therefore, there exists an instance space $\mathcal{X}_d$ for any $H$ such that any self-directed learning algorithm will have a minimum expected regret of $\sqrt{\frac{VC(H)T}{8}}$.
    
\end{proof}



\bibliography{ALT_2024/arxiv_final_version}

\appendix


\newpage

\section{Proofs for Labelling Game and Self-Directed Trees Equivalency, Lemma $\ref{lemma:monotonicity}$, and Lemma $\ref{lemma:SDdim_point}$}
\label{appendix:section_four}

\subsection{Equivalency Between Labelling Game and Self-Directed Trees on $(H, S)$}

\begin{lemma}
    \label{lemma:k_tree_leq_sd_dim}
    For any finite $S \subseteq \mathcal{X}$, label space $\mathcal{Y}$, and concept class $H$, if there exists a realizable $k$-tree, then $SDdim(H, S) \geq k$. 
\end{lemma}

\begin{proof}
Let tree $T$ be a realizable $k$-tree. 
The adversary can use the tree $T$ to setup the labelling game at every single round.
The labelling game sequence will then follow some root-to-leaf path within $T$.
Since $T$ is realizable, then any root-to-leaf path has the guarantee that there exists some $h \in H$ consistent with the entire sequence of points and labels regardless of the learner's choice.
Since the minimum path length in $T$ is $k$, then $SDdim(H, S) \geq k$.
\end{proof}

\begin{lemma}
    \label{lemma:sd_dim_leq_k_tree}
    For any finite $S \subseteq \mathcal{X}$, label space $\mathcal{Y}$, and concept class $H$, $SDdim(H, S)$ is upper bounded by the largest $k \in \mathbb{N} \cup \{0\}$ for which there exists a realizable $k$-tree. 
\end{lemma}

\begin{proof}
We prove this by induction on $k$ with $k = 0$ the base case. 
If the largest realizable $k$-tree on $(H, S)$ is a $0$-tree, then this implies that there are no points in the root node of the tree and all of $S$ is placed in $O_{root}$ with an appropriately chosen label.
If the adversary were to present any point to the learner with two chosen labels, there exists a selection of a point and label that produces an unrealizable sequence.
If this were not true, then one can build a $1$-tree but this violates the assumption so $SDdim(H, S) = 0$.

We now prove that for $k > 0$, if there exists a realizable $k$-tree on $(H, S)$ and $k$ is maximal, then $SDdim(H, S) \leq k$ assuming that it holds for $k - 1$.
In the labelling game, let the adversary select a subset of points $S' \subseteq S$ and two labels for each point to present to the learner. 
This implies that the adversary has selected a single label for each point in $S \setminus S'$.
The learner will compute the largest realizable $k$-tree given $S'$ as the root node with two edges for each point $x \in S'$ corresponding to the two labels chosen by the adversary and $O_{root}$ equivalent to the points in $S \setminus S'$ that the adversary has labelled. 
Then, let us define a learner which selects a point and label that produces at most a $k-1$-realizable tree on the subsequent version space.
Such a point must exist or it would imply that every point and label combination yields a subtree that is a $k$-tree which would imply the existence of a realizable $k+1$-tree violating the assumption.
Let that point and label be represented by $w$ and $y_w$ respectively.
Using the inductive step, then $SDdim(H_{(w, y_w)}, S \setminus \{w \}) \leq k - 1$.
If we suppose the learner plays optimally for the remaining points $S' \setminus \{w\}$,
relative to the corresponding set of consistent concepts $H_{(w,y_w)}$, 
we get that the payout of the game is at most $SDdim(H_{(w,y_w)}, S \setminus \{w\})+1$.
Since the minimax payout is $SDdim(H,S)$, we have that 
$SDdim(H, S) \leq SDdim(H_{(w, y_w)}, S \setminus \{w\}) + 1 \leq k$.
\end{proof}

\begin{lemma}
    For any finite $S \subseteq \mathcal{X}$, label space $\mathcal{Y}$, and concept class $H$, $SDdim(H, S)$ is equal to the size of the largest realizable $k$-tree.
\end{lemma}

\begin{proof}
    Let tree $T$ be the largest realizable $k$-tree where $k$ is maximal.
    From Lemma \ref{lemma:k_tree_leq_sd_dim}, $SDdim(H, S) \geq k$.
    Lemma \ref{lemma:sd_dim_leq_k_tree} tells us that $SDdim(H, S) \leq k$.
    Therefore $SDdim(H, S) = k$.
\end{proof}

\subsection{Lemma $\ref{lemma:monotonicity}$}

\begin{proof}[Lemma~\ref{lemma:monotonicity}]
    We use a proof by contradiction to demonstrate the monotonic nature of $SDdim$.
    Assume that $SDdim(H_{\sigma}, S \setminus \{x_1, ..., x_n \}) > SDdim(H, S)$.
    Let tree $T$ represent some self-directed tree on $(H, S)$ with a minimum root-to-leaf depth of $SDdim(H, S)$.
    Let $O_{root} = \sigma$.
    The version space and current point set of the node are now equivalent to $(H_{\sigma}, S \setminus \{ x_1, ..., x_n \})$.
    One can now construct the remaining part of the self-directed tree $T$ with a minimum path length of $SDdim(H_{\sigma}, S \setminus \{x_1, ..., x_n \})$.
    Therefore, $SDdim(H, S) \geq SDdim(H_{\sigma}, S \setminus \{x_1, ..., x_n \})$.
\end{proof}

\subsection{Lemma $\ref{lemma:SDdim_point}$}

\begin{proof}[Lemma~\ref{lemma:SDdim_point}]
    Assume that $\forall v \in S$, $\exists y_v, y'_v \in \mathcal{Y}, y_v \neq y'_v$ such that $SDdim(H_{(v, y_v)}, S \setminus \{v \}) \geq SDdim(H, S)$ and $SDdim(H_{(v, y'_v)}, S \setminus \{v \}) \geq SDdim(H, S)$.
    Then, one can build a self-directed tree a $T$ that has minimum depth $SDdim(H, S) + 1$ by placing all of $S$ in the the root node of $T$ and for each point $v \in S$, select $y_v, y'_v$ as the two labels.
    Such a construction has the property that the version space on either $(v, y_v)$ or $(v, y'_v)$ has a self-directed tree of minimum depth at least $SDdim(H, S)$, then $T$ has minimum depth $\geq SDdim(H, S) +  1$.
    However, we assumed that the largest self-directed tree on $(H, S)$ has a minimal depth of $SDdim(H, S)$, so there must exist a point $x \in S$ such that at most one label $y_x \in \mathcal{Y}$ has the property that $SDdim(H_{(x, y_x)}, S \setminus \{x \}) \geq SDdim(H, S)$.
\end{proof}

\section{Details on $M_{best}$ vs. $M_{sd}$}
\label{appendix:m_best_proofs}

\subsection{Proof of Theorem $\ref{thm:m_best}$}

\begin{proof}[Theorem~\ref{thm:m_best}]
    If $M_{best}(H) = 1$, this implies that there are at least two different classifiers that labels some point $x \in \mathcal{X}$ differently. 
    From Eq. \ref{eq:all_mistake_bounds}, we know that $M_{sd}(H) \leq 1$.
    To calculate $M_{sd}(H)$, we can first compute $SDdim(H)$ by showing that there exists a finite $S \subseteq \mathcal{X}$ where $SDdim(H, S) = 1$.
    One can build a self-directed tree with $S = \{x\}$ in the root node, and let the two edges correspond to the two different realizable labels.
    Then, the minimal root-to-leaf path length is $1$ so $M_{sd}(H) = SDdim(H) = 1$. 

    To show the other direction, we first assume that $M_{sd}(H) = 1$.
    From Theorem \ref{thm:sddim(H,S)}, we know that if $M_{sd}(H) = 1$, then $M_{SD-SOA}(H) = 1$.
    After the SD-SOA makes its first mistake, any concept still consistent with the sequence of labels and points can be then returned as the true target concept since no more mistakes can be made.
    Therefore, if the best offline algorithm chose the same sequence of points and labels as the SD-SOA assuming the adversary never replied with a mistake, then a single mistake would yield the set of consistent concepts. 
    As a result, $M_{best}(H) = 1$.
\end{proof}

\subsection{Proof of Theorem $\ref{thm:m_best_learnability}$}

\begin{proof}[Theorem~\ref{thm:m_best_learnability}]
    Fix any $n \geq 3$.
    Let the set $\hat{P} = \{ P_1, P_2, ..., P_{(2^n)!} \}$ denote all the possible permutations of the numbers $\{1, ..., 2^n\}$.
    Each $P_i \in \hat{P}$ can be interpreted as a uniquely-ordered list of the index of elements in $\mathcal{X}_n$.
    For example, $P_i(k)$ denotes the index of the element in the $k^{th}$ position. 
    Define $\mathcal{Y}_n = \hat{P} \times \{0, 1\}$.

    To define the concept class $\mathcal{H}_n$, let $\mathcal{H}_n = \bigcup\{ H_1, H_2, ..., H_{(2^n)!} \}$ where each $H_i \in \mathcal{H}_n$ is defined as $H_i = \{h^i_1, h^i_2, ..., h^i_{2^n} \}$ containing $2^n$ classifiers.
    Then, for each $h^i_m \in H_{i}, \forall k, 1 \leq k \leq 2^n, h^i_m(x_k) = (P_i, \mathbbm{1}_{[P_i(k) \leq m]} )$.

    The idea behind this construction is to construct a threshold problem on every possible permutation of the instance space $\mathcal{X}_n$.
    For any label $(P, y) \in \mathcal{Y}_n$, the first coordinate $P$ indicates the ordering of the elements in $\mathcal{X}_n$ and the second coordinate $y$ highlights the threshold classifier label of either $0$ or $1$.
    We intentionally design the concept class $\mathcal{H}_n$ such that for any sequence of instances chosen by the $M_{best}$ learner before the prediction process starts, the adversary can force the learner into a threshold classification problem for which it has chosen a poor ordering.

    Now, we formally analyze the mistake-bound of $M_{best}(\mathcal{H}_n)$ as compared to $M_{sd}(\mathcal{H}_n)$.
    For calculating the self-directed complexity, we can directly compute $SDdim(H_n)$ by playing the labelling game.
    Let $S \subseteq \mathcal{X}_n$ be a set of points chosen by the adversary that is presented to the learner at the beginning of the labelling game.
    Remember that two distinct labels must be chosen for every point in $S$, so let $(P, y) \in \mathcal{Y}_n$, where $P \in \hat{P}$ and $y \in \{0, 1\}$, represent the selected label for the selected point $x \in S$.
    The label $P$ will automatically reveal the ordering of points underlying that threshold classification problem.
    As a result, player B will borrow its strategy for threshold classifiers (see Example \ref{example2}) on the remaining set of points.
    It selects the $x_k \in S \setminus \{x\}$ having \emph{largest} $P(k)$; 
    Player A must offer labels $(P,0), (P,1)$ for this $x_k$ (otherwise, offering some $(P',y)$ with $P' \neq P$ would lead to a game payout $-1$), 
    and player B can then choose label $(P,1)$, 
    at which point player A must label the remaining points $(P,1)$ 
    and thereby end the game (again, to avoid a payout of $-1$).
    The labelling game lasts $2$ rounds, so $M_{sd}(\mathcal{H}_n) = SDdim(\mathcal{H}_n) \leq 2$.
    The fact that $SDdim(\mathcal{H}_n) \geq 2$ as well will follow from 
    the fact that $M_{best}(\mathcal{H}_n) \geq 2$ (established below)
    together with Theorem~\ref{thm:m_best}.

    To calculate $M_{best}(\mathcal{H}_n)$, consider the case $S = \mathcal{X}_n$, and notice that the learner is allowed to choose any sequence of examples $\sigma$ where $\{\sigma[1], \sigma[2], ...\}$ represent the order of examples chosen before the prediction process starts. 
    For the first point in the sequence, $\sigma[1]$, let $(P', y') \in \mathcal{Y}_n$ be the label predicted by the learner.
    The adversary will instantly respond with a mistake by replying with the true label $(P, y)$ where $y \neq y'$, and $P(2^{n-1}) = \sigma[1]$, $P(2^{n-2}) = \sigma[2]$, $P(2^{n-1} + 2^{n-2}) = \sigma[3]$, and so on and so forth.
    In essence, the sequence $\sigma$ is treated as an array representation of a binary tree where the permutation $P$ corresponds to an in-order (i.e., breadth-first search order) traversal of this binary tree (similar to the binary tree construction in $M_{online}$ vs $M_{sd}$). 
    As a result, $\sigma$ corresponds to a series of queries that searches through the midpoints of intervals on $\mathcal{X}_n$ dictated by the ordering $P$, where the intervals become progressively smaller. 
    From here, we may follow a known argument for the mistake bound of learning threshold classifiers when examples are given as a breadth-first traversal of the corresponding Littlestone tree.
    Specifically, the learner has already made one mistake at the root $x_{\sigma[1]}$.
    Depending whether the value $y$ was $0$ or $1$, 
    the adversary will choose the left or right child ($x_{\sigma[2]}$ or $x_{\sigma[3]}$) respectively as the next focus, 
    and when the learner reaches this point and predicts a label $(P',y'')$, 
    the adversary will respond with $(P,y''')$ with $y''' \neq y''$, 
    which shifts the next focus point as the left or right child in the tree,
    based on whether $y'''$ is $0$ or $1$, respectively, and so on.
    Thus, the learner makes (at least) one mistake per depth in the tree, 
    namely on these ``focus'' points which form a root-to-leaf path.
    Since the binary tree is a perfect binary tree of depth $n$, 
    the learner makes at least $n$ mistakes in this process. 
    Thus, $M_{best}(\mathcal{H}_n) \geq n$.
\end{proof}

\subsection{Proof of Corollary \ref{corollary:m_best_vs_m_sd}}

\begin{proof}[Corollary~\ref{corollary:m_best_vs_m_sd}]
    From Theorem $\ref{thm:m_best_learnability}$, we can construct triplets $(\mathcal{X}_n, \mathcal{Y}_n, \mathcal{H}_n)$ for any $n \geq 3$ where $M_{best}(\mathcal{H}_n) \geq n$ and $M_{sd}(\mathcal{H}_n) = 2$ on that space.
The overarching idea in this proof is to combine each of these triplets as disjoint partitions of a bigger space, namely $(\mathcal{X}, \mathcal{Y}, \mathcal{H})$, such that solving the threshold problem on any one partition results in only one remaining classifier realizable on the sequence of points and labels.
The classifiers in $\mathcal{H}_n$ are designed to label points $x \notin \mathcal{X}_n$ a specific label such that only that classifier can be realizable on the remaining sequence of points and labels.
In both $M_{best}$ and $M_{sd}$, it is then in the adversary's advantage not to include points from different threshold problem spaces.
Since $M_{best}(\mathcal{H}, \mathcal{X}_n) \geq n$ and $M_{sd}(\mathcal{H}, \mathcal{X}_n) = 2$ for every $n \geq 3$, then $M_{best}(\mathcal{H}) = \infty$ while $M_{sd}(\mathcal{H}) = 2$.
In the remainder of the proof, we explain in technical detail how we produce such a construction.

For each $n \in \mathbb{N}$ where $n \geq 3$, let $\mathcal{X}_n = \{ (n, x_1), ..., (n, x_{2^n}) \}$ be a set containing $2^n$ unique coordinate pairs.
Then, we apply Theorem $\ref{thm:m_best_learnability}$ on each $\mathcal{X}_n$ to obtain a label set $\mathcal{Y}_n$ and concept class $\mathcal{H}_n$ such that $M_{best}(\mathcal{H}_n) \geq n$ and $M_{sd}(\mathcal{H}_n) = 2$ with respect to $\mathcal{X}_n$.  We specifically consider the definition $\mathcal{H}_n$ from the proof of Theorem~\ref{thm:m_best_learnability}, slightly modified as described below.
We let the instance space $\mathcal{X} = \cup_{n \in \mathbb{N}, n \geq 3} \mathcal{X}_n$ and $\mathcal{H} = \cup_{n \in \mathbb{N}, n \geq 3} \mathcal{H}_n$.
From Theorem $\ref{thm:m_best_learnability}$, each label space $\mathcal{Y}_n = \hat{P} \times \{0, 1\}$ where $\hat{P}$ is the set of all permutations over $\mathcal{X}_n$.
We append an extra set of labels, $\{n \} \times \hat{P} \times \{1,...,2^n \}$, to each $\mathcal{Y}_n$ so that $\mathcal{Y}_n = \hat{P} \times \{0, 1\} \cup ( \{ n\} \times \hat{P} \times \{1,...,2^n \})$.
Then, let $\mathcal{Y} = \cup_{n \in \mathbb{N}, n \geq 3} \mathcal{Y}_n$.

We will now define how the classifiers $h \in \mathcal{H}_n$ classify points $(k, x_j) \notin \mathcal{X}_n$.
As in the proof of Theorem $\ref{thm:m_best_learnability}$, each $\mathcal{H}_n = \bigcup \{H_1, H_2, ..., H_{(2^n)!} \}$ and for every $H_i \in \mathcal{H}_n$, each $H_i = \{h^i_1, h^i_2, ..., h^i_{2^n} \}$.
Additionally, for each $\hat{P}$ corresponding to $\mathcal{X}_n$, $\hat{P} = \{ P_1, P_2, ..., P_{(2^n)!} \}$ where each $P_i \in \hat{P}$ is a permutation of $\mathcal{X}_n$.
For each $h^i_m \in H_i$, and for any $(k, x_j) \notin \mathcal{X}_n$, we define $h^i_m ((k, x_j)) = (n, P_i, m)$ where $(n, P_i, m) \in \mathcal{Y}_n$.
The intention is that, for any $h \in \mathcal{H}_n$, 
for any point $x \notin \mathcal{X}_n$, the label $h(x)$ entirely reveals 
the identity of $h$.

We now shift towards calculating the mistake-bounds $M_{best}(H)$ and $M_{sd}(H)$.
To calculate $M_{best}(H)$, realize that $M_{best}(H) = \max_{S \subseteq \mathcal{X}, |S| < \infty} M_{best}(\mathcal{H}, S)$ and $M_{best}(\mathcal{H}, \mathcal{X}_n) \geq n$ for each $n \geq 3$ (established in the proof of Theorem~\ref{thm:m_best_learnability}).
As a result, $\lim_{n \rightarrow \infty} M_{best}(\mathcal{H}, \mathcal{X}_n) = \infty$ so $M_{best}(\mathcal{H}) = \infty$.
To calculate $M_{sd}(\mathcal{H})$, we turn to the labelling game.
If there exists an $n \in \mathbb{N}$ where the adversarially selected subset $S \subseteq \mathcal{X}_n$, then from Theorem $10$ we know that the labelling game lasts for at most $2$ rounds (noting that, if player A proposes any of the newly-added labels, player B can simply choose that label, and player A must then label the remaining points and end the game, or else face a payout of $-1$). 
If the adversary chooses some $S \subseteq \mathcal{X}$ such that $S \cap \mathcal{X}_n \neq \emptyset$ and $S \cap \mathcal{X}_{n'} \neq \emptyset$ where $n \neq n'$, then 
let $x \in S \cap \mathcal{X}_n$ and $x' \in S \cap \mathcal{X}_{n'}$.
If, after player A labels some points,
one of these remains unlabeled, then player B can choose to label it 
(in this case, it does not matter what label it chooses).
Then if, after another round of labeling by player A, 
the other of these two points is still not labeled, 
player B can choose to label the other point.
Regardless of which player labels the points, 
after at most $2$ rounds, both $x$ and $x'$ have been labeled 
by some $y$ and $y'$, respectively.
If neither of these labels is of the type 
in some $\{n''\} \times \hat{P} \times \{1,\ldots,2^{n''}\}$, 
then already we have a non-realizable labeling, 
and the payout will be $-1$.
Otherwise, at least one of these labels is of that type, 
and therefore there is only one $h \in \mathcal{H}$ 
consistent with that label.
Thus, either the labels already given are already inconsistent with this $h$ 
(yielding a payout of $-1$) 
or else player A must label all remaining points in agreement with this $h$ 
and thus end the game, 
to avoid a payout of $-1$ from player B choosing an inconsistent label 
on its next turn.
Altogether, there were at most $2$ rounds in this case as well.
Thus, $M_{sd}(\mathcal{H}) \leq 2$.  We also know $M_{sd}(\mathcal{H}) \geq 2$ due to $M_{best}(\mathcal{H}) > 1$ together with Theorem~\ref{thm:m_best}).
Therefore, we show a learnability gap on $(\mathcal{X}, \mathcal{Y}, \mathcal{H})$ where $M_{best}(\mathcal{H}) = \infty$ and $M_{sd}(\mathcal{H}) = 2$.
\end{proof}

\section{$k$-Interval Classifiers and Details on Other Learnability Gaps}
\label{appendix:learnability_gaps}

\subsection{$k$-Interval Classifiers}

Let $H^k$ be the class of $k-$interval classifiers over the nonnegative real line, $[ 0, \infty)$, which is set as the instance space $\mathcal{X}$. 
Formally speaking, $H^k = \{h_{(a_1, b_1),...,(a_k, b_k)}: (a_1, b_1),...,(a_k, b_k) \in [ 0, \infty)$ are disjoint, $ a_i < b_i, 1 \leq i \leq k \}$ and $h_{(a_1, b_1),...,(a_k, b_k)} = \mathbbm{1}_{x \in (a_1, b_1) \cup \cdot \cdot \cdot \cup (a_k, b_k)}$ which is $1$ if $x$ lies in any of the $k-$intervals and $0$ elsewhere. 
In the case of a $k-$interval classifier, we will show that the $M_{sd}(H^k) = 2k$ by proving $SDdim(H^k) = 2k$.
To find the self-directed complexity, we will play the labelling game on $H^k$.
Letting $S \subseteq \mathcal{X}$ be any finite subset of points selected by the adversary, the learner will pick the leftmost point and label it a $1$.
After the adversary has labelled some subset of points, the learner will pick the leftmost point and label it a $0$.
The learner will continue this strategy for $2k - 2$ rounds.
After $2k$ rounds, the learner has had the opportunity to secure the location of the $k-$intervals since the alternating labelling scheme determines the boundaries of the intervals.
With the $k-$intervals set after $2k$ rounds, the labelling game must end because if the adversary were to leave any point unlabelled, the learner can create an unrealizable sequence. 
Additionally, we know that $VC(H^k) = 2k$, so $M_{sd}(H^k) = SDdim(H^k) = 2k$.

\subsection{Different Mistake-Bound Models for Online and Offline Learning}

In many ways, the study of self-directed learning was motivated to understand the potential benefits \cite{goldman1994power} one could gain by enhancing the power of the learner in the online learning scenario. 
On the concept class $H$, the term $M(\mathcal{A}, S, \sigma, f^*)$ denotes the number of mistakes algorithm $\mathcal{A}$ will make on the sequence $\sigma$ generated from $S$ given the target concept $f^* \in H$.
In adversarial online learning, the learner doesn't have the ability to choose its instances, so the mistake-bound of online learning is defined in the following way,

\begin{align*}
    M_{online}(\mathcal{A}, H) = \max_{S \subseteq \mathcal{X}} \max_{f^* \in H} \max_{\sigma} M(\mathcal{A}, S, \sigma, f^*) 
\end{align*}

\noindent and it naturally follows that 

\begin{align*}
    M_{online}(H) = \min_{\mathcal{A}} M(\mathcal{A}, H).
\end{align*}

The notion of \textit{offline} learning was introduced by \cite{ben1997online} in an effort to understand how the uncertainty of not knowing the future instances affect the mistake-bound of a learning model. 
From offline learning, two different paradigms were introduced: the \textit{worst sequence} and \textit{best sequence} offline models.
For the worst sequence offline model, an adversary chooses a full sequence of instances before the learner starts making predictions.
In the \textit{best sequence} offline model, the learner is allowed to choose the sequence of instances (without looking at the true labels) before the prediction process starts. 
In both cases, the full sequence of instances is known by the learner before the prediction process starts. 
From a technical standpoint, we now define the mistake-bound learning models for each of the offline models.
Denote by $\mathcal{A}[\sigma]$ an offline learning algorithm whose input is a sequence $\sigma$ which is an ordering of some $S \subseteq \mathcal{X}$.
Let the mistake-bound of an algorithm $\mathcal{A}[\sigma]$ on sequence $\sigma$, given a target concept $f^* \in H$, and some subset $S \subseteq \mathcal{X}$ be

\begin{align*}
    M(H, S, \mathcal{A}[\sigma]) = \max_{f^* \in H} M(S, \mathcal{A}[\sigma], f^*).
\end{align*}

\noindent Then, the worst sequence offline model is then predicated on finding the worst sequence $\sigma$ that maximizes the number of mistakes for even the best algorithm $\mathcal{A}$:

\begin{align*}
    M_{worst}(H) = \max_{S \subseteq \mathcal{X}} M(H, S) = \max_{S \subseteq \mathcal{X}} \max_{\sigma} \min_{\mathcal{A}} M(H, S, \mathcal{A}[\sigma]).
\end{align*}

\noindent To define the mistake-bound model for the best sequence, we note that the learner chooses the sequence so we take the minimum over all sequences $\sigma$ for the best algorithm $\mathcal{A}$:

\begin{align*}
    M_{best}(H) = \max_{S \subseteq \mathcal{X}} M(H, S) = \max_{S \subseteq \mathcal{X}} \min_{\sigma} \min_{\mathcal{A}} M(H, S, \mathcal{A}[\sigma]).
\end{align*}

\noindent As a result, we get the following ordering on the mistake-bounds for any concept class $H
$ and instance space $\mathcal{X}$:

\begin{align*}
    M_{online}(H) \geq M_{worst}(H) \geq M_{best}(H) \geq M_{sd}(H).
\end{align*}

\noindent For a more detailed description, refer to \cite{ben1995self,ben1997online}.

\subsection{$M_{online}$ vs $M_{sd}$} 
\label{appendix:m_online_vs_m_sd}

To describe the learnability gap between $M_{online}$ and $M_{sd}$, we use the concept class of threshold classifiers $H$ on $\mathcal{X} = [0, a)$ where $a \in \mathbb{R}_{>0}$. 
From Example $2$, we know that $M_{sd}(H) = 2$, and we also show that $M_{online}(H) = \infty$.
To calculate $M_{online}(H)$, we note that Littlestone's result \citep{littlestone:88,ben1997online} shows us the equivalence $M_{online}(H) = LD(H)$ where $LD$ stands for the Littlestone dimension.
To compute $LD(H)$, we aim to construct the largest possible perfect binary mistake-tree.
Let the root of the tree correspond to $a/2$ which is the midpoint of the interval $[0, a)$.
Then the left child would contain the value $a/4$ which is the midpoint of the interval $[0, a/2)$.
The right child will have the midpoint of the interval $[a/2, a)$ which is $3a/4$.
Continuing in this fashion, one constructs a perfect binary tree where at each level, the midpoints correspond to increasingly smaller intervals.
Since any interval on the real line with different endpoints will always contain a midpoint, the depth of this mistake-tree is $\infty$ so $M_{online}(H) = LD(H) = \infty$.

\subsection{Query Learning vs $M_{sd}$} Query learning is a topic that has been widely studied \citep{angluin:87,angluin:04,hellerstein:96,chase:20} and is linked to self-directed learning due to their inherent similarities.
However, the subtleties between the two learning models do induce important differences.
Formally, we define $QC_{MQ}$ as the query complexity of the membership querying model, and on the class of singletons $H$ with $\mathcal{X} = \mathbb{N}$, we show that $QC_{MQ}(H) = \infty$ while $M_{sd}(H)=1$.

We know from Example $1$ that $M_{sd}(H)=1$, so we focus on computing $QC_{MQ}(H)$.
For any query $x$ made by the learner, the adversary knows that if it responds with a label of $1$, then only the classifier $\mathbbm{1}_{\{ x \}}$ is consistent with that classification scheme. 
Instead, the adversary will purposefully reply with a label of $0$ due to the fact that there will always be more than one concept consistent with such a labelling scheme. 
This holds $\forall x \in \mathbb{N}$ and since $|\mathbb{N}|$ is countably infinite, then $QC_{MQ}(H) = \infty$ while $M_{sd}(H) = 1$.

\subsection{$M_{worst}$ vs $M_{sd}$}

To describe the learnability gap between $M_{worst}$ and $M_{sd}$, we first explore the relationship between $M_{worst}$ and $M_{online}$.
In Theorem $6$ from \cite{ben1997online}, they proved a result for any concept class $H$, instance space $\mathcal{X}$, and $S \subseteq \mathcal{X}$, $M_{worst}(H, S) = \Omega(\sqrt{log M_{online}(H, S)})$.
From Appendix \ref{appendix:m_online_vs_m_sd}, we know that if $H$ represents the class of threshold classifiers, $M_{online}(H) = \infty$.
Using the result of Theorem $6$ from \cite{ben1997online}, we then show that $M_{worst}(H) = \infty$.
As a result, on the class of threshold classifiers $H$, $M_{worst}(H) = \infty$ while $M_{sd}(H) = 1$.

\subsection{$VC$ vs $M_{sd}$}
\label{appendix:vc_m_sd}

While the $VC(H)$ serves as a lower bound on $M_{sd}(H)$, there exists a notable learnability gap between $M_{sd}$ and $VC$-dimension.
We draw upon Lemma 10 from \cite{ben1995self} which produces an example of a family of concept classes where $VC$-dimension is fixed but $M_{sd}$ scales arbitrarily. 
While this example was studied under $S = \mathcal{X}$, it is actually the case that $M_{sd}(H) = M_{sd}(H, \mathcal{X})$ so giving the entire instance space produces the worst-case scenario in the self-directed case.
Below, we reproduce the concept class from their paper. 
\\

\noindent \textbf{Lemma 10 (restated from \cite{ben1995self}):} For every pair $(d, n)$ of natural numbers such that $d \geq 3$, there exists a concept class $H^d_n$ such that:

\begin{enumerate}
    \item $H^d_n$ is a class of subsets $\{1, ..., (3^n)d \}$.
    \item $|H^d_n| = 3^n 2^d$.
    \item $VC(H^d_n) = d$.
    \item For every point $x$ in the domain of $H^d_n$, and for every label of $x$, the set of concepts in $H^d_n$ which are consistent with this labeling of $x$ contains a copy of $H^d_{n-1}$.
\end{enumerate}

\noindent The concept class $H^d_n$ can be illustrated with the following diagram below:

\begin{center}
\begin{tabular}{ |c|c|c| } 
 \hline
    $H^d_{n-1}$ & $\begin{matrix}
    1 & \cdots & 1\\
    \vdots & \ddots & \vdots \\
    1 & \cdots & 1
    \end{matrix}$ & $\begin{matrix}
    0 & \cdots & 0\\
    \vdots & \ddots & \vdots \\
    0 & \cdots & 0
    \end{matrix}$ \\ 
 \hline
 $\begin{matrix}
    0 & \cdots & 0\\
    \vdots & \ddots & \vdots \\
    0 & \cdots & 0
    \end{matrix}$ & $H^d_{n-1}$ & $\begin{matrix}
    1 & \cdots & 1\\
    \vdots & \ddots & \vdots \\
    1 & \cdots & 1
    \end{matrix}$ \\ 
 \hline
 $\begin{matrix}
    1 & \cdots & 1\\
    \vdots & \ddots & \vdots \\
    1 & \cdots & 1
    \end{matrix}$ & $\begin{matrix}
    0 & \cdots & 0\\
    \vdots & \ddots & \vdots \\
    0 & \cdots & 0
    \end{matrix}$ & $H^d_{n-1}$ \\ 
 \hline
\end{tabular}
\end{center}

\noindent To interpret this matrix, one can let the rows correspond to the functions in the concept class and the columns represent the elements in the instance space.
The value in some cell $(i, j)$ then corresponds to the classification of element $j$ by classifier $i$. 
$H^d_0$ is defined as a $2^d \times d$ matrix that contains all the functions that shatter a set of size $d$ which correspond to $\{ 0, 1 \}$-valued vectors of length $d$.

On this concept class, $H^d_0$ is the largest block that is shattered so $VC(H^d_n) = d$.
According to Corollary 11 in \cite{ben1995self}, the $M_{sd}(H^d_n) = n + d$.
One can also easily calculate the self-directed complexity by computing $SDdim(H^d_n)$ through the labelling game.
Assume $S = \mathcal{X}$.
Notice that for any point labelled by the self-directed learner, either label contains the full subset of some $H^d_{n-1}$ block.
It is then in the adversary's advantage to label points that are in not in the full subset of the remaining $H^d_{n-1}$ block. 
The learner now has to choose some remaining point which lies in that $H^d_{n-1}$ block. 
Similarly, either label of this point contains a full $H^d_{n-2}$ block.
This process keeps repeating until the set of target concepts is narrowed down to some $H^d_0$ block for which it takes $d$ additional rounds of the labelling game before termination. 
This analysis of the labelling game leads us to see that $SDdim(H^d_{n}) = n + d$ so $M_{sd}(H^d_n) = n + d$.

\subsection{$TD$ vs $M_{sd}$}

In this section we discuss the learnability gap between the minimum teaching dimension and $M_{sd}$.
Let $H$ be some concept class, and let $\mathcal{X}$ be the corresponding instance space. 
A teaching set for some target concept $h \in H$ is defined as some subset $S \subseteq \mathcal{X}$ such that $\forall h' \in H, h' \neq h, \exists x \in S$ such that $h'(x) \neq h(x)$. 
The minimum teaching set for some $h \in H$ is defined as the smallest integer $k \in \mathbb{N}$ such that there exists a teaching set $S$ with $|S|=k$.
Formally speaking, we let $TD(H, h)$ be defined as the size of the smallest teaching set for concept $h$.
Then, the teaching dimension, $TD(H)$, is defined as $TD(H) = \max_{h \in H} TD(H, h)$.

While the teaching dimension acts as a lower bound for the self-directed complexity, one can demonstrate a learnability gap between the two where the teaching dimension is constant but the self-directed complexity is arbitrarily large.
Let $d \geq 3$, $n \geq 0$, $H = H^d_n$ and $\mathcal{X}$ correspond to the concept class and instance space described in Lemma 10 by \cite{ben1995self} illustrated in Appendix \ref{appendix:vc_m_sd}.
To calculate $TD(H)$, note that each $h \in H$ lies in some group of $H^d_0$ classifiers shattered by some set of $d$ points. 
Take some $h \in H^d_0$ and let $S$ correspond to the $d$ points shattered by $H^d_0$.
Now, it's clear that $\forall h' \in H^d_0, h' \neq h$, $\exists x \in S$ such that $h'(x) \neq h(x)$. 
If we want to extend the concept class to $H^d_n$, we must be careful with the selected $h$.
If $h$ corresponds to either $\{1\}^d$ or $\{0\}^d$ on the $d$ points, then it will coincide with another column block that is either all 0's or 1's.
In this scenario, one additional point is needed to differentiate the classifiers. 
Therefore, it follows that $TD(H) = d + 1$ whereas $M_{sd}(H) = n + d$.

\section{Proof of Theorem $\ref{thm:vc_one}$}
\label{appendix:vc_one}

Let $\mathcal{X}$ be any set of either finite or infinite cardinality.
The concept class $H$ will contain functions $h: \mathcal{X} \rightarrow \{0, 1\}$ that are identified with the set $h^{-1}(1) = \{x \in \mathcal{X}: h(x) = 1\}$.
On $\mathcal{X}$, we define a partially ordered set or a tree as $(\mathcal{X}, \preceq)$ such that for any $y \in \mathcal{X}$, the initial segment $\{x \in \mathcal{X}: x \preceq y \}$ is well-ordered.
From \citet{ben-david:15}, we state the following two statements which are equivalent:
\begin{enumerate}
    \item $VCdim(H) = 1$.
    \item There exists some tree ordering over $\mathcal{X}$ such that for any $h \in H$, for some $y \in \mathcal{X}$, if $h(y)=1$, then $h^{-1}(1) = \{x \in \mathcal{X}: x \preceq y \}$
\end{enumerate}

\noindent \citet{goldman1994power} proved that if a concept class $H$ had $VC(H) = 1$, then $M_{sd}(H) = 1$.
We showcase a simple proof for the same result using our dimension.\\

\begin{proof}[Theorem~\ref{thm:vc_one}]
Since $M_{sd}(H) \geq VC(H)$ holds for any $H$ and $M_{sd}(H) = 1$ implies that there exists a point that can be either labeled $0$ or $1$, then $VC(H) = 1$.
We now prove the other direction.
Let some $H$ with $VC(H)=1$ be represented as partial tree ordering over $\mathcal{X}$ under the relation $\preceq$.
Now, we demonstrate a strategy to show that $M_{sd}(H) = 1$ by playing the labelling game.
The game initializes with the adversary picking some $S \subseteq \mathcal{X}$ which can be represented under a partial tree ordering as shown by \citet{ben-david:15}. 
Then, the learner intentionally chooses the greatest element, call it $x$, by picking some leaf node in the tree and labelling it a $1$.
This implies that the initial segment $I_x$, containing all elements that are $\preceq x$, must correspond to some $h \in H$.
More specifically, for some $h \in H$, $I_x = h^{-1}(1)$.
Now, the adversary realizes that if any point in $I_x$ is labelled $0$, this would immediately lead to an unrealizable sequence.
Additionally, if any other point not in $I_x$ is labelled a $1$, then this would produce an unrealizable sequence because no such $h \in H$ is consistent with such a labelling scheme.
As a result, the adversary can opt to label all the remaining points or let the learner produce an unrealizable sequence with its next selection.
In either case, the payout of the game is $SDdim(H) = 1$, so $M_{sd}(H) = 1$ when $VC(H) = 1$.
\end{proof}

\end{document}